\documentclass{article}

\usepackage[preprint]{neurips_2026}

\usepackage[utf8]{inputenc} % allow utf-8 input
\usepackage[T1]{fontenc}    % use 8-bit T1 fonts
\usepackage{hyperref}       % hyperlinks
\usepackage{url}            % simple URL typesetting
\usepackage{booktabs}       % professional-quality tables
\usepackage{amsfonts}       % blackboard math symbols
\usepackage{nicefrac}       % compact symbols for 1/2, etc.
\usepackage{microtype}      % microtypography
\usepackage{xcolor}         % colors
\usepackage[most]{tcolorbox}

\hypersetup{
    hidelinks,        
    colorlinks=true,  
    citecolor=blue,   
    linkcolor=blue   
}
\usepackage{amsmath}
\usepackage{floatflt}
\usepackage{bm}
\usepackage{graphicx}
\usepackage{colortbl} 
\usepackage{enumitem} 
\usepackage{wrapfig}
\usepackage{listings}
\usepackage{algorithmic}  
\usepackage{multirow} 
\usepackage{amssymb}
\usepackage{pifont} 
\usepackage[most]{tcolorbox}
\usepackage{newfloat}
\usepackage{listings}
\usepackage{caption}
\usepackage{enumitem}
\usepackage[ruled,vlined]{algorithm2e}

\definecolor{morandiblue}{RGB}{145,168,176}
\definecolor{morandired}{RGB}{196,150,148}
\definecolor{mydarkred}{RGB}{170,30,30} 
\definecolor{mydarkblue}{RGB}{30,30,170} 

\usepackage{arydshln} 
\newtheorem{theorem}{Theorem}[section]

\usepackage{bbm}

\title{APPO: Agentic Procedural Policy Optimization}

\author{%
  Xucong Wang\thanks{\parbox[t]{0.85\linewidth}{%
    Equal Contribution. \\ \textsuperscript{\textdagger}Project lead: Yong Wang; Corresponding authors: Yong Wang and Pengkun Wang}}\quad 
  Ziyu Ma$^{*}$\quad 
  Yong Wang$^{\dagger}$\quad 
  Yuxiang Ji\quad 
  Shidong Yang\quad \\ 
  \textbf{Guanhua Chen\quad   
  Pengkun Wang$^{\dagger}$\quad 
  Xiangxiang Chu}\\
  AMAP, Alibaba Group 
}

\begin{document}

\maketitle

\begin{abstract}
Recent advances in agentic Reinforcement Learning (RL) have substantially improved the multi-turn tool-use capabilities of large language model agents. However, most existing methods assign credit over coarse heuristic units, such as tool-call boundaries or fixed workflows, making it difficult to identify which intermediate decisions influence downstream outcomes. In this work, we study agentic RL from two perspectives: \textit{where to branch and how to assign credit after branching}. Our pilot analysis shows that influential decision points are broadly distributed throughout the generated sequence rather than concentrated at tool calls, while token entropy alone does not reliably reflect their impact on final outcomes. Motivated by these observations, we propose \textbf{Agentic Procedural Policy Optimization (APPO)}, which shifts branching and credit assignment from coarse interaction units to fine-grained decision points in the sequence. APPO selects branching locations using a Branching Score that combines token uncertainty with policy-induced likelihood gains of subsequent continuations, enabling more targeted exploration while filtering out spurious high-entropy positions. It further introduces procedure-level advantage scaling to better distribute credit across branched rollouts. Experiments on 13 benchmarks show that APPO consistently improves strong agentic RL baselines by nearly 4 points, while keeping efficient tool-calls and maintaining behavior interpretability. Project Page: \href{https://github.com/AMAP-ML/APPO}{\textbf{Github}}. 
\end{abstract}

\section{Introduction}
Large language models (LLMs)~\cite{wei2022chain,jaech2024openai,team2025kimi1,team2025kimi,team2025longcat,yang2025qwen3,li2025system} have evolved from static text generators~\cite{radford2019language} into autonomous agents~\cite{wang2025toward,zhang2025landscape,dong2025agentic,feng2025retool,ma2026skillclaw} capable of multi-turn interaction with external environments, enabling strong performance in long-horizon~\cite{lu2025pilotrl,gao2025beyond} real-world tasks~\cite{hafner2021benchmarking,cao2025skyrl,shridhar2020alfworld,yao2022webshop}. 
This progress is largely driven by Reinforcement Learning with Verifiable Rewards (RLVR), which enables policy optimization using sparse outcome-level supervision~\cite{guo2025deepseek,schulman2017proximal,lee2023rlaif,rafailov2023direct,shao2024deepseekmath,chu2026gpg}. 
However, this training paradigm introduces a fundamental limitation: feedback is only provided at the trajectory level, making it difficult to attribute success or failure to specific intermediate decisions. As a result, each trajectory provides only a coarse and entangled learning signal, leading to inefficient credit assignment and unstable policy improvement~\cite{qian2025smart,feng2025group,hou2025treerl,ji2025tree}.

To address this limitation, existing approaches restructure agent rollouts to extract more informative credit signals under limited rollout budgets. A common strategy is to expand trajectories from intermediate locations, constructing multiple candidate branches and assigning credit according to outcome differences across rollouts~\cite{yao2022react,zhou2023language,ji2025tree,zhao2026training,shen2025carl}. The intuition is that, instead of repeatedly sampling full trajectories with substantial redundancy, one can branch around uncertain positions and compare alternative continuations to identify which local decisions lead to better final outcomes. Depending on the branching granularity, this line of work includes workflow-level exploration and tool-call-level branching~\cite{hou2025treerl,dong2025agentic,dong2025agentic1,li2025deepagent}.

While such designs improve rollout efficiency and partially densify supervision from sparse outcome rewards, they still rely on coarse-grained units for credit assignment. In particular, they often compress the entire non-tool-call process into \texttt{<thinking>} tags~\cite{dong2025agentic,dong2025agentic1} or rely on fixed workflows~\cite{yao2022react,wang2023plan,shinn2023reflexion}, thereby overlooking the procedural knowledge embedded within the thinking content~\cite{guo2025segment,wang2025beyond,wu2026procedural}. In practice, successful long-horizon agent reasoning is often shaped not by entire thinking blocks or workflow stages, but by a small number of critical \emph{decision points} where alternative continuations lead to substantially different downstream outcomes~\cite{wang2025beyond}. These decision points are latent positions in the trajectory, instantiated as tokens in the generated sequence, and identified by their role in inducing divergence in subsequent reasoning paths. We use \textbf{procedures} to denote the procedural reasoning patterns organized around such high-impact decision points. Although some individual procedures, such as \texttt{plan}~\cite{wang2023plan} and \texttt{reflect}~\cite{shinn2023reflexion}, have drawn attention in training-free prompt engineering~\cite{jin2025hira,wang2026textbf}, their collective role in online agentic RL remains underexplored.

\begin{figure*}[t] 
    \centering
    \includegraphics[width=1\linewidth]{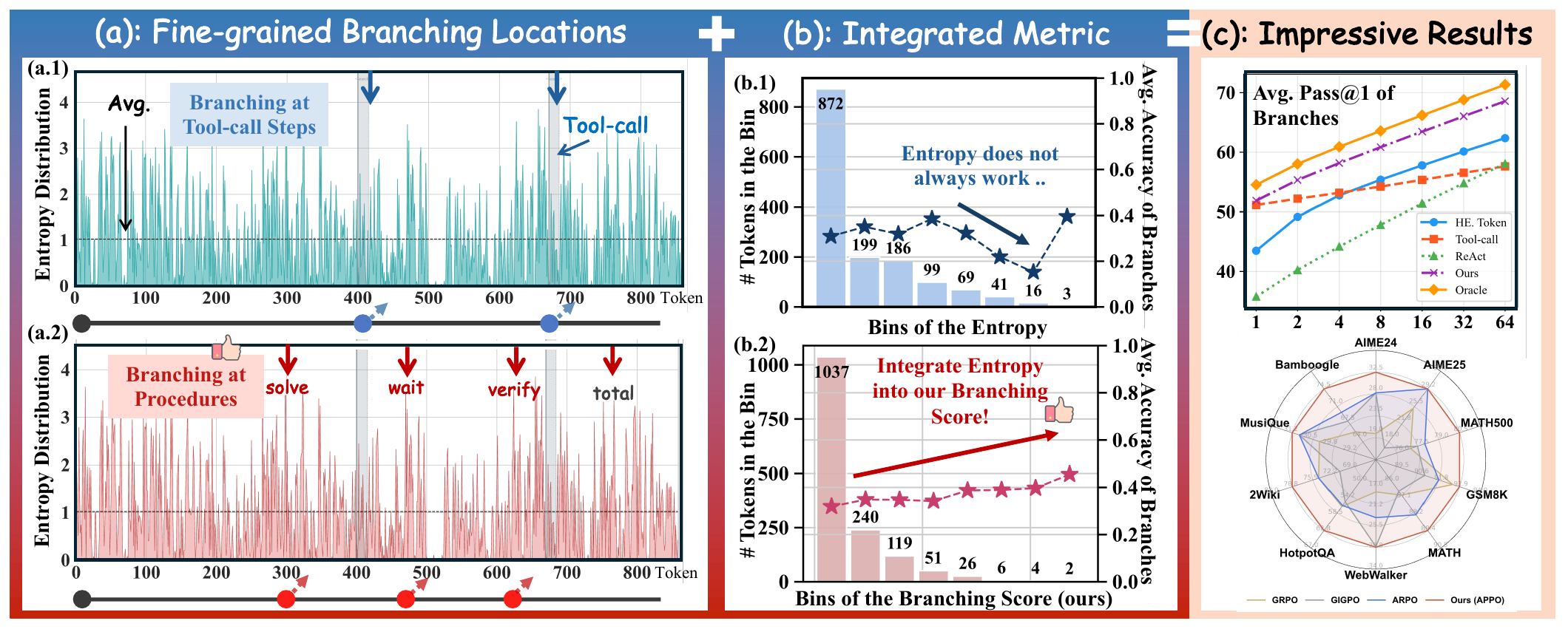} 
    \vspace{-0.4cm}
    \caption{\textbf{(a):} The token entropy distribution in the tool-integrated rollout (sampled from Tool-Star's~\cite{dong2025tool} 54K dataset). \textbf{(b):} Average accuracy of branches generated from each token, shown by bins of the entropy and the APPO's Branching Score (BS). \textbf{(c):} The pass@$k$ of rollouts resampled via different criteria (``oracle'' means to resample from the points with the highest accuracy uncertainty); The performance comparison between APPO and others on 10 datasets.} 
    \vspace{-0.5cm}
    \label{fig:1}
\end{figure*} 

To further investigate how these procedures relate to reasoning accuracy and failure, we conduct a pilot study summarized in Figure~\ref{fig:1}, where we analyze \textbf{(a)} branching locations and \textbf{(b)} the average accuracy of resampled branches under different branching criteria. The results reveal that neither coarse tool-call boundaries nor raw token entropy provides a satisfactory basis for credit assignment. Specifically, Figure~\ref{fig:1}(a) shows that the highest-uncertainty positions are not concentrated at tool-call boundaries, but are broadly distributed throughout the thinking span, suggesting that non-tool-call reasoning contains finer-grained procedural information beyond coarse tool-call units. Moreover, Figure~\ref{fig:1}(b.1) shows that high token entropy alone does not reliably indicate decision significance, as tokens with higher entropy do not consistently yield branches with greater outcome uncertainty, implying that some entropy peaks reflect lexical rarity rather than significance to the task outcome. 

 \begin{tcolorbox}[
  colback=gray!10,     
  colframe=black,      
  boxrule=1pt,        
  arc=2mm,           
  left=4pt,right=4pt,top=4pt,bottom=4pt, 
]
 \textit{These findings highlight the necessity of attending to procedural information within the thinking process beyond tool calls, while underscoring the need for a more principled design of branching criteria to enable finer-grained credit assignment across procedures.}
\end{tcolorbox}

Motivated by these findings, we propose Agentic Procedural Policy Optimization (APPO), an agentic RL algorithm that redefines the unit of credit assignment from coarse-grained heuristic units to fine-grained procedures. Given an initial rollout, APPO extends branching-point selection from tool-call boundaries to the entire sequence, and chooses branching tokens using a comprehensive Branching Score (BS). Beyond token entropy, BS measures how much the current policy increases the likelihood of subsequent continuations relative to the old policy, thereby capturing the future value carried by the current token and filtering out spurious high-entropy positions. Building on this design, we further introduce a procedure-level advantage scaling term based on $\Omega$ to encourage exploration over procedures with high branching value. We evaluate APPO on 13 challenging benchmarks spanning deep information seeking, knowledge-intensive reasoning, and computational problem solving, and show that it consistently improves both task performance and exploration flexibility over strong baselines (Figure~\ref{fig:1}(c)). Our contributions are threefold:  

\vspace{-0.2cm}
\begin{itemize}[leftmargin=*] 
    \item Our preliminary study demonstrates the critical role of \textit{procedures} in agent reasoning. By moving rollout branching from workflows or tool-call boundaries to procedures, it exposes finer-grained structure within the thinking process and enables more informative intermediate supervision.
    \vspace{-0.1cm}
    \item Motivated by these findings, we propose APPO, an agentic RL algorithm that shifts credit assignment from coarse-grained heuristic units to fine-grained procedures. Its Branching Score combines token entropy with the policy-induced likelihood gain of subsequent continuations, enabling the selection of high-value procedures for branching and procedure-level advantage scaling.
    
    \item Extensive experiments validate the effectiveness of our method, which outperforms existing approaches by approximately 3 points across 13 benchmarks, while achieving comparable tool-calls and maintaining interpretability.
\end{itemize}

\section{Related Work}
\noindent\textbf{Agentic Reinforcement Learning.} Reinforcement Learning~\cite{kaufmann2023survey,schulman2017proximal,rafailov2023direct,guo2025deepseek} has proven to be effective in endowing agents with long-horizon complex reasoning and acting capabilities.  Beginning from the actor-critic-based PPO~\cite{schulman2017proximal}, a wide range of studies aim to devise more efficient policy variants, or refine the scale of policy gradients, advantages or regularization terms, like GRPO~\cite{guo2025deepseek}, DAPO~\cite{yu2025dapo}, GSPO~\cite{zheng2025group}, GPG~\cite{chu2026gpg}, Dr. GRPO~\cite{liu2025understanding}. More recent research~\cite{feng2025group,zhai2025agentevolver,lu2026skill0}  focus on the distinct applications of RL in agent tasks, such as fine-grained credit assignment for actions~\cite{feng2025group,dong2025agentic,dong2025agentic1}, agent self-evolution~\cite{wu2025evolver,zhai2025agentevolver,xia2026skillrl} and internalization of agent skills~\cite{lu2026skill0,zhang2026memskill}.

\noindent\textbf{Tree-based RL.}
Tree-based RL can be roughly divided into the following categories: (i) Offline Training~\cite{feng2023alphazero,lai2024step,li2025iterative,xie2024monte}, where  branches split from the same point are stored as offline preference data and later utilized in DPO~\cite{rafailov2023direct}. For example, MCTS-DPO~\cite{xie2024monte} compares the estimated reward of candidate nodes and separates them into chosen / rejected samples; SPORT~\cite{li2025iterative} leverages step-wise sampling and automatic verification to formulate step-wise preference data for agent training. (ii) Online Training~\cite{hou2025treerl,ji2025tree,dong2025agentic,dong2025agentic1,zhao2026training}, where the branches are extended from rollouts in each step and used to calculate group-relative advantages. For example, Tree-GRPO~\cite{ji2025tree} randomly selects think-action steps for branching, while ARPO~\cite{dong2025agentic} identifies the high-entropy tokens following each tool-call and conducts resampling accordingly. (iii) Test-Time Scaling~\cite{wang2026textbf,snell2024scaling,wu2024inference,welleck2024decoding,muennighoff2025s1}, where the model generates multiple branches to elevate the average performance in inference. Notably, the design of branching location is a critical factor in elevating the performance ceiling of tree-based RL.  \textit{In contrast with these lines of studies, APPO treats procedures as finer units of rollout branching in on-policy agentic RL and elicits the reasoning process accordingly. }

 \section{APPO: Agentic Procedural Policy Optimization}

\subsection{Preliminaries}

\noindent\textbf{Agentic Reinforcement Learning.}
We consider a standard agentic RL setting~\cite{gou2023tora,li2025torl,wu2025agentic,dong2025tool,su2025toolorchestra,qian2025toolrl}, where an agent receives a task $x\in\mathcal{X}$ and interacts with a toolset $T$. The rollout consists of $T_a$ interleaved thinking and tool-call steps, followed by $T_b$ answer generation steps:
\begin{equation}
P_\theta(\mathcal{G},y|x,T)=\prod_{t=1}^{T_a}[\pi_\theta(\mathcal{O}_{t}|\mathcal{G}_{<t},x;T)P_{env}(\mathcal{G}_t|\mathcal{O}_{t},\mathcal{G}_{<t})]\prod_{t=1}^{T_b}\pi_\theta(\mathcal{O}_{t}|y_{<t},\mathcal{G},x;T)
\end{equation}
where $P_{env}$ denotes the external environment. If the agent does not call the external tool-use actions in its outputs $\mathcal{O}_{t}$, then $\mathcal{O}_{t}$ equals the overall output $\mathcal{G}_t$ of the step $t$. Based on the rollouts, the training objective of Agentic Reinforcement Learning process can be formulated as follows:
\begin{equation}
    {\rm max}_{\pi_\theta}\mathbb{E}_{x\sim\mathcal{X},y\sim\pi_\theta(\cdot|x;T)}[r_\phi(x,y)]-\beta\,\mathbb{D}_{\rm KL}[\pi_\theta(y|x;T)||\pi_{\rm ref}(y|x;T)]
\end{equation}
where $\pi_\theta$ and $\pi_{\rm ref}$ are the policy and reference models, respectively. $r_\phi$ is the reward function, and $\mathbb{D}_{\rm KL}$ denotes KL divergence. In practice, rewards are converted into advantages $\hat{A}(\cdot)$ through group-level~\cite{shao2024deepseekmath,guo2025deepseek,liu2025understanding}, sequence-level~\cite{zheng2025group}, token-level~\cite{yu2025dapo}, or value-function-based estimators~\cite{schulman2017proximal}.

\vspace{0.3em}
\noindent\textbf{Token Entropy.}
A common branching strategy~\cite{dong2025agentic,dong2025agentic1,hou2025treerl} selects tokens with the top-$k$ entropy:
\begin{equation}
    H_t=-\sum\nolimits_{j=1}^{|\mathcal{V}|}p_{t,j}\log p_{t,j}, \quad \bm{p}_t=\pi_\theta(\cdot|\mathcal{G}_{<t},z;T)={\rm Softmax}(\bm{z}_t/\tau)
\end{equation}
where $|\mathcal{V}|$ is the vocabulary size, $\tau$ is the temperature, and $\bm{z}_t$ are logits. Although entropy reflects local uncertainty, it does not indicate whether a token corresponds to a decision point that changes downstream reasoning. Thus, high-entropy tokens may arise from lexical uncertainty rather than task-relevant procedural choices. APPO addresses this by identifying tokens that instantiate latent decision points, i.e., positions where alternative continuations induce divergent reasoning paths, and uses them to construct procedures for more targeted branching and credit assignment.

\begin{figure*}[t] 
    \centering
    \includegraphics[width=1\linewidth]{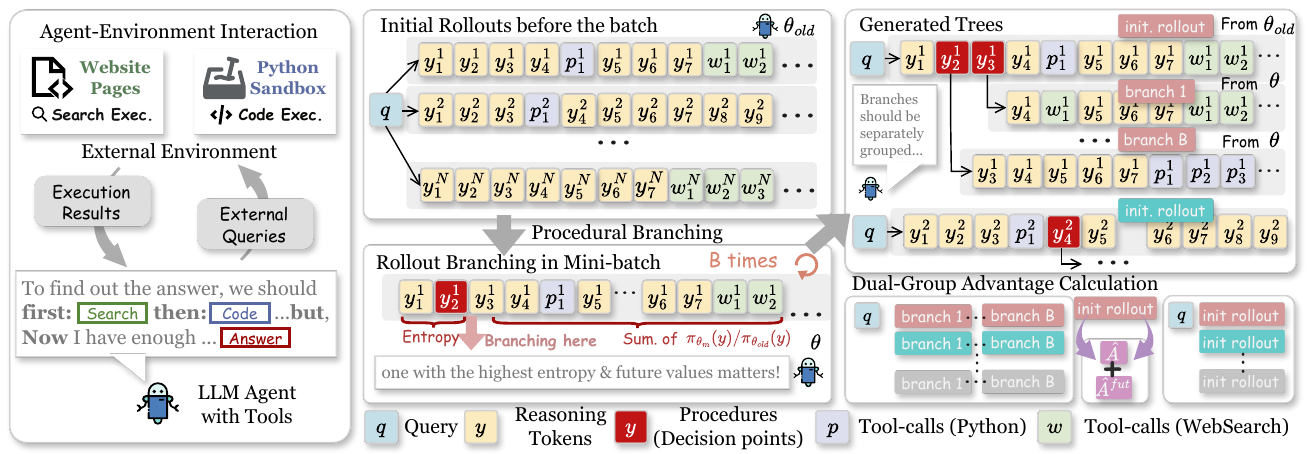} 
    \vspace{-0.4cm}
    \caption{Overview of APPO. The agent first interacts with the environment to generate initial rollouts for each batch. During mini-batch training, APPO identifies fine-grained decision points using the Branching Score (BS), which combines token entropy with future-aware likelihood gains, and resamples continuations from these positions rather than fixed tool-call boundaries. The resulting branches and initial rollouts are then used for dual-group advantage estimation, together with a future-aware advantage term for procedure-level credit assignment.}
    \vspace{-0.3cm}
    \label{fig:2}
\end{figure*}

\subsection{Procedural Rollout Branching}

In this section, we introduce the APPO algorithm in detail, designed to guide LLMs in exploring fine-grained procedural-level behaviors under synthetic branching criteria, as illustrated in Figure \ref{fig:2}. 

\textbf{$\blacktriangleright$ Initialization.}
Given input $x$ and a global rollout budget $M$, the model first generates $N$ full rollouts $\mathcal{T}=\{\mathcal{H}_n|\mathcal{H}_n\sim\pi_{\theta}(\cdot|x)\}^N$, which serve as the roots of $N$ independent trees.

\textbf{$\blacktriangleright$ Sampling.}
For each token $\mathcal{H}_{n,i}$ in rollout $\mathcal{H}_n$, APPO measures its future-aware influence using the accumulated decayed importance sampling ratio ($\pi_{\rm old}$ here is the $\pi_{\theta}$ in initialization):
\begin{equation}
    \Omega_{n,i}=\exp\Big(\sum_{i'\ge i}\gamma^{i'-i}\log \rho_{i'}(\theta)\Big), \quad 
    \rho_{i'}(\theta)=\frac{\pi_\theta(\mathcal{H}_{n,i'}|\mathcal{H}_{n,<i'},x;T)}{\pi_{\rm old}(\mathcal{H}_{n,i'}|\mathcal{H}_{n,<i'},x;T)},
\end{equation}
where $\pi_{\rm old}$ denotes the policy used to generate the initial rollouts. We name $\Omega_{n,i}$ as the future value. A larger $\Omega_{n,i}$ indicates that the model’s continuations are assigned states that are more favored by training, and vice versa. In this case, we treat $\Omega_{n,i}$ as a replacement of the posterior accuracy-uncertainty, which can only be calculated via innumerable times of Monte-Carlo estimation. Also, to mitigate gradient and variance fluctuations caused by accumulation, we introduce a discount factor $\gamma$ to suppress the influence of tokens farther away from the current token, thereby reducing variance. We then combine the future-value term with token entropy to define the Branching Score:
\begin{equation}
    {\rm BS}_{n,i}= \mathcal{Z}({\rm clip}(\Omega_{n,i},1-\epsilon',1+\epsilon');\mathcal{H}_{n}) \cdot \mathcal{Z}(H_{n,i};\mathcal{H}_{n}),
\end{equation}
where $\mathcal{Z}(\cdot;\mathcal{H}_n)$ denotes z-score normalization within rollout $\mathcal{H}_n$. Entropy captures local uncertainty, while $\Omega_{n,i}$ captures future influence; their product selects tokens that are both uncertain and consequential, serving as a sufficient integration of the sequence prior and posterior.

For each rollout $\mathcal{H}_n$, we select the top-$B$ tokens according to ${\rm BS}_{n,i}$ and denote them as $\mathcal{B}_n$ ($|\mathcal{B}_n|=B$). These tokens instantiate latent decision points. We then resample continuations from each $\mathcal{H}_{n,i}\in\mathcal{B}_n$ to generate new branches $\mathcal{H}_n^{new}$, and update the rollout tree as $\mathcal{H}_n \leftarrow \mathcal{H}_n \cup \mathcal{H}_n^{new}$.

\textbf{$\blacktriangleright$ Termination.}
Tree expansion stops when either the remaining budget $M-N$ is exhausted or no further branching is performed. It is worth noting that we sample uniformly across all trees. Assume the total number of generation loops is $L$, then for $L=1$ (i.e., the source of branching is restricted to the initial rollouts), there is $(B+1)\cdot N=M$. With a fixed total budget, increasing $N$ allows us to obtain rollouts with a more diverse initial distribution; conversely, increasing $B$ and $L$ enables more branching on a single tree, thereby providing finer-grained credit assignment. We will conduct detailed ablation studies of these coefficients in subsequent sections and Appendix \ref{appx:e}.

\subsection{Procedural Advantage Estimation and Policy Optimization}

Most policy optimization methods~\cite{ji2025tree,xin2025bfs} rely on group-level advantages, which provide limited credit assignment at intermediate decision points. APPO instead assigns credit to procedure-level decisions by using branches as auxiliary contrastive signals. Unlike most Tree RL methods, our branches are generated by the current mini-batch policy $\pi_\theta$ rather than $\pi_{\rm old}$. Since gradients are not propagated through all branches, we use them only to compute rewards and advantages.

To avoid bias from mixing rollouts generated by different policies, APPO computes group-relative advantages separately for initial rollouts $\mathcal{T}_{init}$ and branches $\mathcal{T}_{branch}$:
\begin{equation}\label{eq:dual}
    \hat{A}^{\rm base}_{n,i} = {\rm avg}\left\{
    \frac{R_n - {\rm mean}(\{R_{n'} \mid \mathcal{H}_{n'} \in \mathcal{T}_*\})}
    {{\rm std}(\{R_{n'} \mid \mathcal{H}_{n'} \in \mathcal{T}_*\})}
    \;\middle|\;
    \mathcal{T}_* \in \{\mathcal{T}_{init}, \mathcal{T}_{branch}\}
    \right\}
\end{equation}
where $R_n$ is the task reward. Since generated tokens serve as observable instantiations of latent decision points, token-level advantages can be viewed as localized credit assigned to the corresponding procedural decisions. Furthermore, inspired by recent studies~\cite{huang2026direction,meng2026sparse}, we argue that critical procedures offer more ``sparse and critical'' subsequences; these procedures serve as the turning point of the rollout and more likely to induce higher policy-induced differences over continuations. We then take the similar design of $\Omega$ to formulate an extra advantage term, which further emphasizes decisions with stronger downstream influence by assigning larger credits:
\begin{equation}\label{eq:fut}
    \hat{A}^{\rm fut}_{n,i} = {\rm clip}_{\epsilon'}\left(
    \exp\Big(\sum\nolimits_{i'\ge i}\gamma^{i'-i}\log \rho_{i'}(\theta)\Big)
    \right), \quad 
    \rho_{i'}(\theta)=\frac{\pi_\theta(\mathcal{H}_{n,i'}|\mathcal{H}_{n,<i'}, {x};T)}
    {\pi_{\rm old}(\mathcal{H}_{n,i'}|\mathcal{H}_{n,<i'}, {x};T)}
\end{equation}
where ${\rm clip}_{\epsilon'}$ clips the value into $[1-\epsilon', 1+\epsilon']$. The final advantage is:
$\hat{A}_{n,i}=\hat{A}^{\rm base}_{n,i}(1+b\cdot\hat{A}^{\rm fut}_{n,i})$,
where $b$ controls the contribution of the future-aware term. Finally, APPO optimizes:
\begin{equation}\label{eq:obj} 
    J(\theta)\!=\!\mathbb{E}_{\substack{x\sim\mathcal{X} \\  \mathcal{H}\sim\pi_{\rm old}(\cdot|x)}}\!\!\left[
    \frac{1}{M}\!\sum_{n=1}^N\frac{1}{|\mathcal{H}_n|}\sum_{i=1}^{|\mathcal{H}_n|}
    {\rm min}\big(\rho_{n,i}(\theta)\hat{A}_{n,i}, {\rm clip}_\epsilon(\rho_{n,i}(\theta))\hat{A}_{n,i}\big)
    \!\!-\!\beta\!\,\mathbb{D}_\mathrm{KL}(\pi_\theta||\pi_{\rm ref})
    \right]
\end{equation}
where ${\rm clip}_{\epsilon}$ clips the term into $[1-\epsilon, 1+\epsilon]$. $\pi_{\rm ref}$ and $\pi_{\rm old}$ denote the reference and behavior policies, respectively, and $\beta$ controls the KL regularization strength. Branches sampled from $\pi_\theta$ are not directly optimized; they provide auxiliary procedural signals for advantage estimation.

\subsection{Theoretical Foundation of APPO}

We provide a theoretical foundation showing that branching at decision points reduces gradient variance and that the future-aware advantage design admits a policy improvement bound. First, APPO motivates high-impact decision points via BS, leading to the following variance reduction property:
\vspace{-0.1cm}
\begin{theorem}[Variance Reduction]\label{theo:1}
Let $g_{\mathrm{APPO}}$ denote the gradient estimator guided by ${\rm BS}$ at decision points, and let $g_{\mathrm{base}}$ denote the estimator using random branching under the same computational budget. Suppose the conditional reward variance of a decision point $D_i$ is monotone in its branching score ${\rm BS}_i$, i.e., $\mathrm{Var}(R\mid D_i)=f({\rm BS}_i)$ with $\nabla f(\cdot)\ge 0$; Then, with $\sigma_i^2 := \mathrm{Var}(R \mid D_i) = f({\rm BS}_i)$ and branches are allocated as $n^{\rm APPO}_i= M \cdot \sigma_i \big/ \sum_{j=1}^K \sigma_j$. APPO can allocate more samples to high-variance decision points, and the resulting estimator satisfies $\mathrm{Var}(g_{\mathrm{APPO}})\le \mathrm{Var}(g_{\mathrm{base}})-\Delta_\Omega({\rm BS})$, where $\Delta_\Omega({\rm BS})\ge 0$ denotes the variance reduction induced by branching score.  
\end{theorem}
\vspace{-0.1cm}
We defer the proof to Appendix \ref{proof:theorem1}. Beyond variance reduction, we further introduce Theorem~\ref{theo:2}, which shows that APPO's advantage design (Eq.\ref{eq:fut})  admits a policy improvement bound under procedural branching:
\vspace{-0.1cm}
\begin{theorem}[Policy Improvement Bound]\label{theo:2}
Assume $\mathrm{KL}(\pi_{\rm new}\,\|\,\pi_{\rm old})\le \epsilon$ and $1-\epsilon' \le \hat{A}^{\rm fut}(s)\le 1+\epsilon'$ for all visited states $s$. Let $\mathcal{J}(\pi)$ denote the expected return, $\omega(s)=1+b\,\hat{A}^{\rm fut}(s)$,
and $q$ the state distribution induced by APPO's BS-guided branching mixture. Then
\begin{equation}
    \mathcal{J}(\pi_{\rm new}) - \mathcal{J}(\pi_{\rm old})
    \ge \frac{1}{1-\gamma}\,
    \mathbb{E}_{s \sim \rho^{q}, a \sim \pi_{\rm new}}
    \left[ \omega(s) A^{\pi_{\rm old}}(s, a) \right]
    - \frac{C\,\epsilon}{(1-\gamma)^2},
\end{equation}
where $C$ depends on $r_{\max}$, $b$, and $\epsilon'$.
\end{theorem}

Theorem~\ref{theo:2} indicates that APPO provides a valid policy improvement bound under ${\rm BS}$-guided branching. Detailed proof is provided in Appendix~\ref{appx:a}. Together, these results support decision points as a principled unit for agent exploration and credit assignment.

\begin{table*}[t]
\caption{Performance comparison between APPO and others on 10 challenging Deep Reasoning datasets. Notably the best and suboptimal results are in \textbf{bold} and \underline{underlined} respectively. }
\centering
\label{tab:1}
% \vspace{-0.2cm}
\renewcommand{\arraystretch}{1.1}   
\setlength\tabcolsep{4pt} 
\resizebox{1\columnwidth}{!}{
\begin{tabular}{cccccccccccc}
\toprule[1pt]
\multicolumn{1}{c}{\multirow{2}{*}{\raisebox{-0.8ex}{\textbf{Method}}}} & \multicolumn{5}{c}{\textbf{Mathematical Reasoning}} & \multicolumn{5}{c}{\textbf{Knowledge-Intensive Reasoning}} & \multirow{2}{*}{\raisebox{-0.8ex}{\textbf{Avg.}}} \\ \cmidrule(lr){2-6}\cmidrule(lr){7-11}
\multicolumn{1}{c}{} & AIME24 & AIME25 & MATH500 & GSM8K & \multicolumn{1}{c}{MATH} & \multicolumn{1}{l}{WebWalker} & \multicolumn{1}{l}{HotpotQA} & \multicolumn{1}{l}{2Wiki} & \multicolumn{1}{l}{Musique} & \multicolumn{1}{l}{Bamboogle} &  \\ \hline
\multicolumn{12}{c}{\textbf{\textit{Backbone: Llama3.1-8B-Instruct}}} \\ \hline
\multicolumn{12}{l}{\textit{\textbf{Training-Free Method}}} \\
\multicolumn{1}{l}{Zero-shot} & 3.3 & 0.0 & 43.3 & 81.4 & \multicolumn{1}{c}{60.6} & 3.0 & 24.3 & 24.6 & 10.4 & \multicolumn{1}{c}{40.0} & 28.8 \\
\multicolumn{1}{l}{TIR Prompting} & 3.3 & 3.3 & 39.4 & 73.8 & \multicolumn{1}{c}{58.2} & 15.0 & 48.5 & 47.5 & 15.5 & \multicolumn{1}{c}{58.4} & 36.3 \\ \hdashline
\multicolumn{12}{l}{\textbf{\textit{Classic RL Method}}} \\
\multicolumn{1}{l}{GRPO} & 13.3 & \underline{13.3} & 62.4 & 87.4 & \multicolumn{1}{c}{79.2} & 26.5 & 57.8 & 71.8 & 31.0 & \multicolumn{1}{c}{68.2} & 51.1 \\
\multicolumn{1}{l}{Reinforce++} & 13.3 & \textbf{16.7} & 61.4 & 87.0 & \multicolumn{1}{c}{77.2} & 27.5 & 57.1 & 71.6 & 29.9 & \multicolumn{1}{c}{69.1} & 51.1 \\
\multicolumn{1}{l}{DAPO} & 16.7 & \underline{13.3} & 61.2 & 87.4 & \multicolumn{1}{c}{76.4} & 25.5 & 56.6 & 70.3 & 29.2 & \multicolumn{1}{c}{67.3} & 50.4 \\
\multicolumn{1}{l}{GPPO} & 16.7 & 6.7 & 61.8 & 86.6 & \multicolumn{1}{c}{79.4} & 27.5 & \underline{61.8} & 72.8 & 29.8 & \multicolumn{1}{c}{71.9} & 51.5 \\
\multicolumn{1}{l}{CISPO} & 13.3 & 6.7 & 62.2 & 87.0 & \multicolumn{1}{c}{78.2} & 26.0 & 57.3 & \underline{75.6} & 32.2 & \multicolumn{1}{c}{71.8} & 51.0 \\ \hdashline
\multicolumn{12}{l}{\textbf{\textit{Agentic RL Method}}} \\
\multicolumn{1}{l}{GIGPO} & 20.0 & \underline{13.3} & 62.4 & 87.4 & \multicolumn{1}{c}{77.2} & \underline{31.5} & \underline{61.8} & 74.6 & 31.8 & \multicolumn{1}{c}{72.1} & 53.2 $\Delta_{\rm base}$\\
\multicolumn{1}{l}{ARPO} & \underline{23.3} & \textbf{16.7} & \underline{64.6} & \underline{88.0} & \multicolumn{1}{c}{\underline{80.2}} & 30.5 & \textbf{65.4} & 75.5 & \underline{34.8} & \multicolumn{1}{c}{\underline{72.8}} & \underline{55.3} \textcolor[rgb]{0,0.4,0}{\tiny $+3.9\%$}\\ 
\rowcolor{gray!20}
\multicolumn{1}{l}{\textbf{APPO (Ours)}} & \textbf{30.0} & \textbf{16.7} & \textbf{69.4}  & \textbf{88.6} & \multicolumn{1}{c}{\textbf{81.4}} & \textbf{32.0} & \textbf{65.4} & \textbf{79.8} & \textbf{35.2} & \multicolumn{1}{c}{\textbf{76.8}} & \textbf{57.4} \textcolor[rgb]{0,0.4,0}{\tiny $+7.9\%$}\\ \hline
\multicolumn{12}{c}{\textbf{\textit{Backbone: Qwen2.5-7B-Instruct}}} \\ \hline
\multicolumn{12}{l}{\textit{\textbf{Training-Free Method}}} \\
\multicolumn{1}{l}{Zero-shot} & 10.0 & 10.0 & 70.6 & 90.2 & \multicolumn{1}{c}{82.0} & 2.0 & 12.2 & 12.6 & 6.6 & \multicolumn{1}{c}{24.0} & 32.0 \\
\multicolumn{1}{l}{TIR Prompting} & 6.7 & 10.0 & 68.2 & 64.6 & \multicolumn{1}{c}{78.2} & 15.5 & 14.8 & 18.3 & 9.5 & \multicolumn{1}{c}{23.6} & 30.9 \\ \hdashline
\multicolumn{12}{l}{\textbf{\textit{Classic RL Method}}} \\
\multicolumn{1}{l}{GRPO} & 23.3 & \underline{26.7} & 78.0 & 92.8 & \multicolumn{1}{c}{87.8} & 22.0 & 59.0 & \underline{76.1} & 30.6 & \multicolumn{1}{c}{68.4} & 56.5 \\
\multicolumn{1}{l}{Reinforce++} & 26.7 & 23.3 & 78.0 & \underline{92.2} & \multicolumn{1}{c}{\underline{88.8}} & 26.0 & 55.1 & 68.9 & 25.2 & \multicolumn{1}{c}{64.9} & 54.9 \\
\multicolumn{1}{l}{DAPO} & 20.0 & 23.3 & \underline{80.4} & 91.0 & \multicolumn{1}{c}{\underline{88.8}} & 24.0 & 57.7 & 68.4 & 28.6 & \multicolumn{1}{c}{65.5} & 54.8 \\
\multicolumn{1}{l}{GPPO} & 26.7 & 23.3 & \underline{80.4} & 91.0 & \multicolumn{1}{c}{\underline{88.8}} & \underline{31.0} & \underline{60.7} & 74.2 & \textbf{31.5} & \multicolumn{1}{c}{\underline{72.4}} & 58.0 \\
\multicolumn{1}{l}{CISPO} & 26.7 & \textbf{30.0} & 77.8 & 91.4 & \multicolumn{1}{c}{86.2} & 29.0 & 59.3 & 72.1 & 29.1 & \multicolumn{1}{c}{70.2} & 57.2 \\ \hdashline
\multicolumn{12}{l}{\textbf{\textit{Agentic RL Method}}} \\
\multicolumn{1}{l}{GIGPO} & \underline{30.0} & 20.0 & 78.4 & 91.6 & \multicolumn{1}{c}{87.6} & {30.5} & 58.1 & 73.5 & \underline{31.1} & \multicolumn{1}{c}{70.1} & 57.1 $\Delta_{\rm base}$ \\
\multicolumn{1}{l}{ARPO} &  \underline{30.0} & \textbf{30.0} & 78.8 & \underline{92.2} & \multicolumn{1}{c}{\underline{88.8}} & 26.0 & 58.8 & \underline{76.1} & \underline{31.1} & \multicolumn{1}{c}{{71.5}} & \underline{58.3} \textcolor[rgb]{0,0.4,0}{\tiny $+2.1\%$} \\
\rowcolor{gray!20}
\multicolumn{1}{l}{\textbf{APPO (Ours)}} & \textbf{36.7} & \textbf{30.0}  & \textbf{81.0}  & \textbf{93.2} & \multicolumn{1}{c}{\textbf{90.0}} & \textbf{33.5} & \textbf{66.5} & \textbf{81.5} & \textbf{31.5} & \multicolumn{1}{c}{\textbf{77.6}} & \textbf{62.2}  \textcolor[rgb]{0,0.4,0}{\tiny $+8.9\%$}  \\ \bottomrule[1pt]
\end{tabular}}\vspace{-0.4cm}
\end{table*}

\begin{table*}[t]
\caption{Performance comparison between APPO and others on 4 challenging DeepSearch datasets. Results with \textcolor{gray}{gray} color are from larger / closed-source models and only for reference. Notably the best and suboptimal results are in \textbf{bold} and \underline{underlined} respectively.}
\centering
\label{tab:2}
% \vspace{-0.2cm}
\renewcommand{\arraystretch}{1.1}   
\setlength\tabcolsep{7pt} 
\resizebox{1\columnwidth}{!}{
\begin{tabular}{lccccccccccccc}
\toprule[1pt]
\multicolumn{1}{c}{\multirow{2}{*}{\textbf{Method}}} & \multicolumn{4}{c}{\textbf{General AI Assistant}} & \multicolumn{4}{c}{\textbf{WebWalkerQA}} & \multicolumn{4}{c}{\textbf{Humanity's Last Exam}} & \textbf{Xbench}   \\ \cmidrule(lr){2-5}\cmidrule(lr){6-9}\cmidrule(lr){10-13}\cmidrule(lr){14-14} 
\multicolumn{1}{c}{} & Lv.1 & Lv.2 & Lv.3 & Avg. & Easy & Med. & Hard & Avg. & NS & CE & SF & Avg. & Avg.  \\ \hline
Qwen3-32B-thinking & 26.2 & 12.1 & 0.0 & 14.9 & 6.9 & 1.1 & 2.9 & 3.1 & 14.6 & 9.8 & 8.4 & \underline{12.6} & 14.0  \\
DeepSeek-R1-32B & 21.5 & 13.6 & 0.0 & 14.2 & 7.5 & 1.4 & 4.2 & 3.8 & 6.6 & 5.1 & 6.5 & 6.4 & 10.0  \\
QwQ-32B & 30.9 & 6.5 & 5.2 & 18.9 & 7.5 & 2.1 & 4.6 & 4.3 & 11.5 & 7.3 & 5.2 & 9.6 & 10.7   \\
GPT-4o & 23.1 & 15.4 & 8.3 & 17.5 & 6.7 & 6.0 & 4.2 & 5.5 & 2.7 & 1.2 & 3.2 & 2.6 & 18.0   \\
DeepSeek-R1-671B & \textcolor{gray}{40.5} & \textcolor{gray}{21.2} & \textcolor{gray}{5.2} & \textcolor{gray}{25.2} & \textcolor{gray}{5.0} & \textcolor{gray}{11.8} & \textcolor{gray}{11.3} & \textcolor{gray}{10.0} & \textcolor{gray}{8.5} & \textcolor{gray}{8.1} & \textcolor{gray}{9.3} & \textcolor{gray}{8.6} & \textcolor{gray}{32.7}  \\
o1-preview & \textcolor{gray}{/} & \textcolor{gray}{/} & \textcolor{gray}{/} & \textcolor{gray}{/} & \textcolor{gray}{11.9} & \textcolor{gray}{10.4} & \textcolor{gray}{7.9}& \textcolor{gray}{9.9}& \textcolor{gray}{12.9} & \textcolor{gray}{8.1}& \textcolor{gray}{6.6} & \textcolor{gray}{11.1} & \textcolor{gray}{/}   \\ \hline
\multicolumn{14}{l}{\textbf{\textit{Single-Enhanced Method (Qwen3-8B)}}} \\ 
Vanilla RAG & 28.2 & 15.4 & \textbf{16.7} & 20.4 & 8.9 & 10.7 & 9.9 & 10.0 & 5.1 & 1.6 & 12.9 & 5.8 & 8.0   \\
Search-o1 & 35.9 & 15.4 & 0.0 & 21.4 & 6.7 & 15.5 & 9.7 & 11.5 & 7.6 & 2.7 & 5.3 & 6.4 & 10.0   \\
WebThinker & 43.6 & 11.5 & 0.0 & 22.3 & 6.7 & 13.1 & 16.9 & 13.0 & 7.3 & 4.0 & 6.3 & 6.6 & 13.0   \\
ReAct & 35.9 & 17.3 & \underline{8.3} & 23.3 & 8.9 & 16.7 & \underline{18.3} & 15.5 & 4.2 & 4.0 & 6.3 & 4.6 &  16.0   \\ \hdashline
\multicolumn{14}{l}{\textbf{\textit{RL-based Method (Qwen3-8B)}}} \\ 
Qwen3-8B & 28.1 & 15.4 & \textbf{16.7} & 20.4 & 0.0 & 2.4 & 2.8 & 2.0 & 3.9 & 2.7 & 8.4 & 4.6 & 9.0   \\
\,+GRPO &  {48.7} &  {25.0} & \underline{8.3} & 32.0 & 28.9 & 32.1 & \textbf{28.2} & 30.0 & \underline{7.9} & 4.0 & 10.5 & 7.8 & 20.0   \\
\,+ARPO & \underline{53.9} & \underline{32.7} & \textbf{16.7} & \underline{38.8} & \underline{31.1} & \underline{35.7} & \textbf{28.2} & \underline{32.0} & 7.3 & \underline{6.7} & \textbf{15.8} & \underline{8.8} & \underline{25.0}   \\ 
\rowcolor{gray!20}
\,+\textbf{APPO (Ours)} & \textbf{59.0} & \textbf{36.5} & \textbf{16.7} & \textbf{42.7} & \textbf{37.5}  & \textbf{36.4} & \textbf{28.2} & \textbf{33.8} &\textbf{14.7} & \textbf{8.3 }&\underline{13.0}  &\textbf{13.4}  & \textbf{28.0}  \\ \hline
\multicolumn{14}{l}{\textbf{\textit{Single-Enhanced Method (Qwen3-14B)}}} \\ 
Vanilla RAG & 38.5 & 19.2 & 8.3 & 25.2 & 17.8 & 13.1 & 11.3 & 13.5 & 5.5 & 6.3 & 9.4 & 6.0 & 15.0   \\
Search-o1 & 48.7 & 23.1 & 0.0 & 30.1 & 11.1 & 21.4 & 16.9 & 17.5 & 6.4 & 4.0 & 10.5 & 6.8 & 21.0   \\
WebThinker & 48.7 & 26.9 & 8.3 & 33.0 & 13.3 & 23.8 & 18.3 & 19.5 & 7.0 & 4.0 & 9.5 & 7.0 & 23.0   \\
ReAct & 48.7 & 25.0 & 8.3 & 32.0 & 11.1 & 20.2 & 12.7 & 15.5 & 5.8 & 5.3 & 10.5 & 6.6 & 20.0  \\ \hdashline
\multicolumn{14}{l}{\textbf{\textit{RL-based Method (Qwen3-14B)}}} \\ 
Qwen3-14B & 33.3 & 13.5 & 0.0 & 19.4 & 6.7 & 2.4 & 4.2 & 4.0 & 5.5 & 6.7 & 11.6 & 6.8 & 14.0  \\
\,+GRPO & 51.3 & 34.6 & 0.0 & 36.9 & 35.6 & 42.9 & 35.2 & 38.5 & 7.9 & 6.7 & 12.6 & 8.6 & 27.0   \\
\,+ARPO & \underline{56.4} & \underline{40.4} & \underline{16.7} & \underline{43.7} & \underline{40.0} & \underline{44.1} & \underline{36.6} & \underline{40.5} & \underline{10.3} & \underline{10.7} & \underline{13.7} & \underline{10.0} & \underline{32.0}   \\ 
\rowcolor{gray!20}
\,+\textbf{APPO (Ours)} & \textbf{61.5} & \textbf{42.3} & \textbf{16.7} & \textbf{46.6} & \textbf{43.0}  & \textbf{48.6} & \textbf{37.6} & \textbf{43.4} &\textbf{11.5} & \textbf{12.8} &\textbf{14.5}  &\textbf{12.3}  & \textbf{34.0}  \\ \bottomrule[1pt]
\end{tabular}}\vspace{-0.4cm}
\end{table*}	 	 	 	 	 	 	 	 	 	 	 	 

\section{Experiments}
\subsection{Experiment Setup}
\noindent\textbf{Datasets.} Following~\cite{dong2025agentic1,dong2025agentic}, we employ  the following three kinds of benchmarks to comprehensively evaluate the model's capability in long-term complex tool usage: \textit{(1) Mathematical Reasoning,} including the typical math problems like GSM8K~\cite{cobbe2021training} and MATH~\cite{hendrycks2021measuring}; and competitive math problem like AIME24~\cite{aime24}, AIME25~\cite{aime25} and MATH500~\cite{lightman2023let}; \textit{(2) Knowledge-Intensive Reasoning,} including four multi-hop question-answering problems based on Wikipedia, i.e., HotpotQA~\cite{yang2018hotpotqa}, 2WikiMultihopQA~\cite{ho2020constructing}, Musique~\cite{trivedi2022musique} and Bamboogle~\cite{press2023measuring}. We also include the web multi-hop task WebWalker~\cite{wu2025webwalker}; \textit{(3) Deep Search,} including the challenging General AI Assistant (GAIA)~\cite{mialon2023gaia}, Humanity's Last Exam~\cite{phan2025humanity}, and information retrieval tasks like WebWalker~\cite{wu2025webwalker} and Xbench~\cite{chen2025xbench}.

\noindent\textbf{Baselines.} Four kinds of baselines are selected for comprehensive comparison: \textit{(1) Vanilla RL Method,} including GRPO~\cite{guo2025deepseek}, Reinforce++~\cite{hu2025reinforce++}, DAPO~\cite{yu2025dapo}, GPPO~\cite{su2025klear} and CISPO~\cite{chen2025minimax}; \textit{(2) Agentic RL Method}, including GIGPO~\cite{feng2025group} and ARPO~\cite{dong2025agentic}; \textit{(3) LLM Backbones,} where we adopt the instruct version of Llama3.1-8B~\cite{dubey2024llama} and Qwen2.5~\cite{hui2024qwen2} as the generalized backbones; we also report the results of QwQ~\cite{team2024qwq}, DeepSeek-R1~\cite{guo2025deepseek}, GPT-4o~\cite{hurst2024gpt}, o1-preview~\cite{hurst2024gpt} and Qwen3-32B~\cite{yang2025qwen3}; \textit{(4) Search Agents and others,} where we adopt the Search-o1~\cite{li2025search} and Webthinker~\cite{li2025webthinker}; For other mechanisms we employ the RAG~\cite{lewis2020retrieval} and ReAct~\cite{yao2022react} framework.

\noindent\textbf{Metrics.} We employ the F1-score to evaluate the performance of models over four question-answering tasks. For other tasks, we employ  LLM-as-a-Judge and deploy a Qwen2.5-72B-Instruct model with vLLM for correctness calculation. We use pass@1 over the LLM-judged results.   

\noindent\textbf{Implementation Details.}  Following~\cite{ma2026fipo}, we implement the decay rate in procedure advantage as $\gamma=2^{-\frac{1}{\tau}}$ and keeps $\tau=32$. The overall methodology is implemented on the VeRL framework~\cite{sheng2025hybridflow} where tool-call results are detached from gradient computation.  The batch-size is set to 128 and a PPO mini-batchsize of 16, resulting in 8 gradient updates in one step. We also set the coefficient $\beta=0$ to stabilize training. We train APPO in both reasoning and search tasks following ARPO~\cite{dong2025agentic}, where models are separately trained for 2 and 5 epochs. The search results are the top-10 snippets retrieved by Bing search engine following~\cite{jin2025flashrag,li2025retrollm,dong2025tool} and python code is executed in a sandbox environment. To ensure the model possesses basic tool-use capabilities, we directly adopt the SFT pipeline from ARPO~\cite{dong2025agentic}. More implementation details are attached to Appendix \ref{appx:c}. 

\subsection{Main Results}
\noindent\textbf{Results on Mathematical Reasoning and Knowledge-Intensive Reasoning.} 
As shown in Table~\ref{tab:1}, APPO consistently achieves superior performance across both task groups. 
We make the following observations: 
(i) On mathematical reasoning tasks, APPO surpasses all baselines, improving over the previous best agentic RL method by an average of 2.45 points. This suggests that branching at high-impact decision points, rather than fixed turn-level boundaries, enables the model to better explore and reinforce critical deductive paths for complex problem solving. (ii) On knowledge-intensive reasoning tasks, APPO ranks first on nearly all datasets, showing strong performance on multi-hop information synthesis tasks. This indicates that identifying and assigning higher credit to high-impact decision points helps the model strengthen useful procedural reasoning patterns, such as planning, reflection, and verification, leading to consistent gains over both classical RL and prior agentic RL methods that rely on coarser credit assignment.

\noindent\textbf{Results on DeepSearch Tasks.} 
The results in Table~\ref{tab:2} demonstrate both the effectiveness and efficiency of APPO in complex deep search scenarios. While even powerful large-scale closed-source models achieve unsatisfactory performance, APPO establishes new state-of-the-art results at both the 8B and 14B model scales, consistently outperforming strong baselines such as ARPO. For instance, on the GAIA benchmark, APPO achieves scores of 42.7 and 46.6 with Qwen3-8B and Qwen3-14B, respectively, while also improving performance on HLE. These gains are particularly notable given that deep search tasks involve long horizons and intricate tool-use patterns. The advantage of APPO stems from its decision-point-based branching mechanism: by focusing on high-impact positions that induce divergence in reasoning trajectories, APPO avoids excessive branching on consecutive high-entropy but low-impact tokens, while encouraging more diverse and effective search paths. 
Consequently, by prioritizing structurally meaningful decision points, APPO enables more targeted exploration and improves the reliability of result-oriented agentic reasoning.

\begin{figure*}[t] 
    \centering
    \includegraphics[width=1\linewidth]{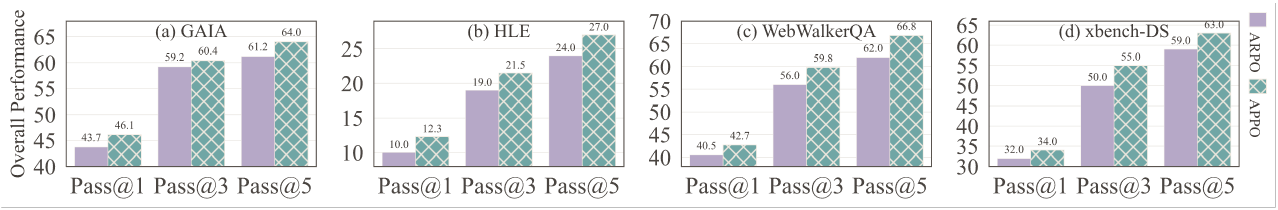} 
    \vspace{-0.5cm}
    \caption{Analysis of Pass@1 to Pass@5 of ARPO and APPO on four datasets respectively.} 
    \vspace{-0.5cm}
    \label{fig:passk}
\end{figure*}

\begin{table}[t]
\centering
\begin{minipage}[t]{0.50\textwidth}
\centering
\caption{Analysis on branching config $(L=1)$.}\label{tab:budget}
\vspace{0.15cm}
\renewcommand{\arraystretch}{1.14}   
\setlength\tabcolsep{2pt} 
\resizebox{1\columnwidth}{!}{
\begin{tabular}{lcccccc}
\toprule[1pt]
\textbf{Method} & \textbf{Web.} & \textbf{Hotpot.} & \textbf{2Wiki} & \textbf{Musiq.} & \textbf{Bamb.} & \textbf{Avg.} \\ \hline
ARPO $(M=16)$ & \multicolumn{1}{c}{26.0} & \multicolumn{1}{c}{58.8} & \multicolumn{1}{c}{76.1} & \multicolumn{1}{c}{31.1} & \multicolumn{1}{c}{71.5} & 52.7 \\ \hline
\multicolumn{7}{c}{$M=4$} \\ \hline
APPO $(N=2, B=1)$ & 25.4 & 60.4 & 77.0 & 30.6 & 70.2 & 52.7 \\
APPO $(N=1, B=3)$ & 26.7 & 63.8 & 77.6 & 31.0 & 70.6 & 53.9 \\ \hline
\multicolumn{7}{c}{$M=8$} \\ \hline
APPO $(N=2, B=3)$ & 28.0 & 65.0 & 79.4 & 30.8 & 73.3 & 55.3 \\ \hline
\multicolumn{7}{c}{$M=16$} \\ \hline
\rowcolor{gray!20}
APPO $(N=4, B=3)$ & \textbf{33.5} & \multicolumn{1}{c}{\textbf{66.5}} & \multicolumn{1}{c}{\textbf{81.5}} & \multicolumn{1}{c}{\textbf{31.5}} & \multicolumn{1}{c}{77.6} & \textbf{58.1} \\
APPO $(N=8, B=1)$ & 31.8 & 65.8 & 81.2 & 31.0 & \textbf{79.5} & 57.9 \\
APPO $(N=2, B=7)$ & 29.2 & 65.5 & 79.0 & 31.2 & 75.7 & 56.1 \\ \bottomrule[1pt]
\end{tabular}}
\end{minipage}
\hfill 
\vspace{-0.1cm}
\begin{minipage}[t]{0.46\textwidth}
\centering
\caption{Ablation study on components.}\label{tab:comp}
 \vspace{0.15cm}
\renewcommand{\arraystretch}{1.07}   
\setlength\tabcolsep{2pt} 
\resizebox{1\columnwidth}{!}{
\begin{tabular}{lcccccc}
\toprule[1pt]
\textbf{Method} & \textbf{Web.} & \textbf{Hotpot.} & \textbf{2Wiki} & \textbf{Musiq.} & \textbf{Bamb.} & \textbf{Avg.} \\ \hline
\multicolumn{7}{l}{\textbf{\textit{Llama3.1-8B-Instruct}}} \\ 
\rowcolor{gray!20}
APPO & \textbf{32.0} & \textbf{65.5} & \textbf{79.8} & \textbf{35.2} & \textbf{76.8} & \textbf{57.9} \\
BS$\rightarrow$Ent. & 31.4 & 65.4 & 78.0 & 34.8 & 75.5 & 57.0 \\ 
w/o $\hat{A}^{\rm fut}$ & 29.8 & 64.8 & 78.6 & 35.0 & 74.5 & 56.5 \\ 
w/o Dual-group & 30.9 &	65.0 	&76.4 &	35.5 &	76.2 &	56.8 \\
\hline
\multicolumn{7}{l}{\textbf{\textit{Qwen2.5-7B-Instruct}}} \\ 
\rowcolor{gray!20}
APPO & \textbf{33.5} & \textbf{66.5} & \textbf{81.5} & \textbf{31.5} & \textbf{77.6} & \textbf{58.1} \\ 
BS$\rightarrow$Ent. & 29.4 & 64.8 & 80.5 & 31.0 & 75.7 & 56.3 \\ 
w/o $\hat{A}^{\rm fut}$ & 30.3 & 60.0 & 79.9 & 30.6 & 72.6 & 54.7 \\
w/o Dual-group & 33.0 &	63.8 &	75.8 &	30.6 &	76.6 &	56.0  \\
\bottomrule[1pt]
\end{tabular}}
\end{minipage}
\vspace{-0.6cm}
\end{table}

\begin{figure*}[t] 
    \centering
    \includegraphics[width=1\linewidth]{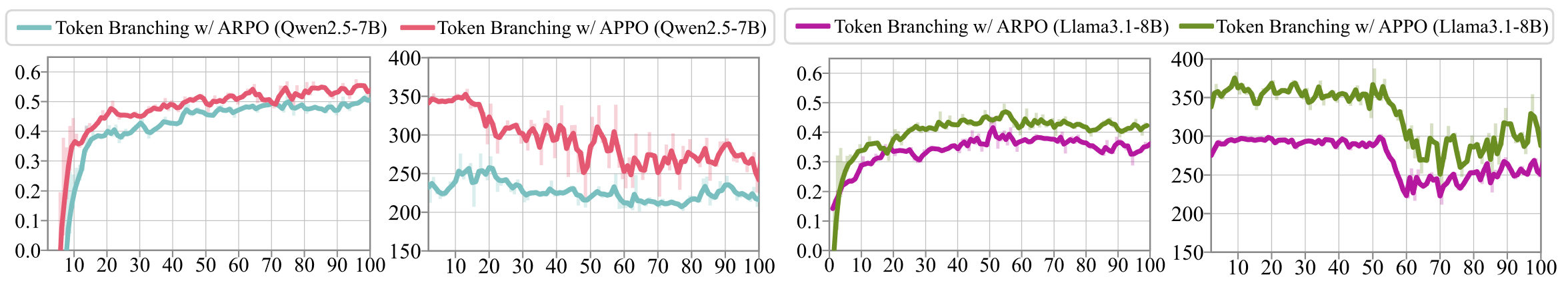} 
    \vspace{-0.7cm}
    \caption{\textbf{The training dynamics} of pure-token branching and APPO's procedural guided branching.} 
    \vspace{-0.4cm}
    \label{fig:dyna}
\end{figure*}

\begin{table}[t]
\centering
\begin{minipage}[t]{0.5\textwidth}
\centering
\includegraphics[width=1\linewidth]{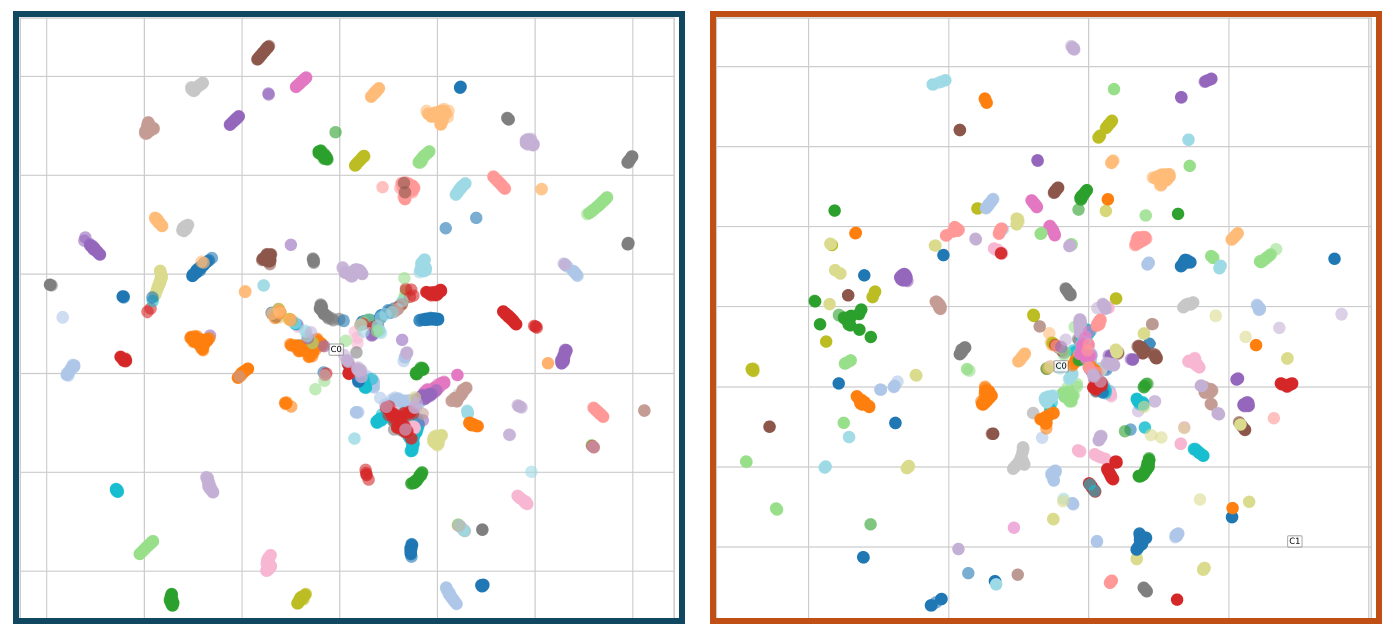}
\captionof{figure}{The visualization of the branch distributions from ARPO \textbf{(left)} and APPO \textbf{(right)}.} 
\label{fig:julei} 
\end{minipage}
\hfill % 填充水平空间
\vspace{-0.1cm}
\begin{minipage}[t]{0.48\textwidth}
\centering

\includegraphics[width=0.95\linewidth]{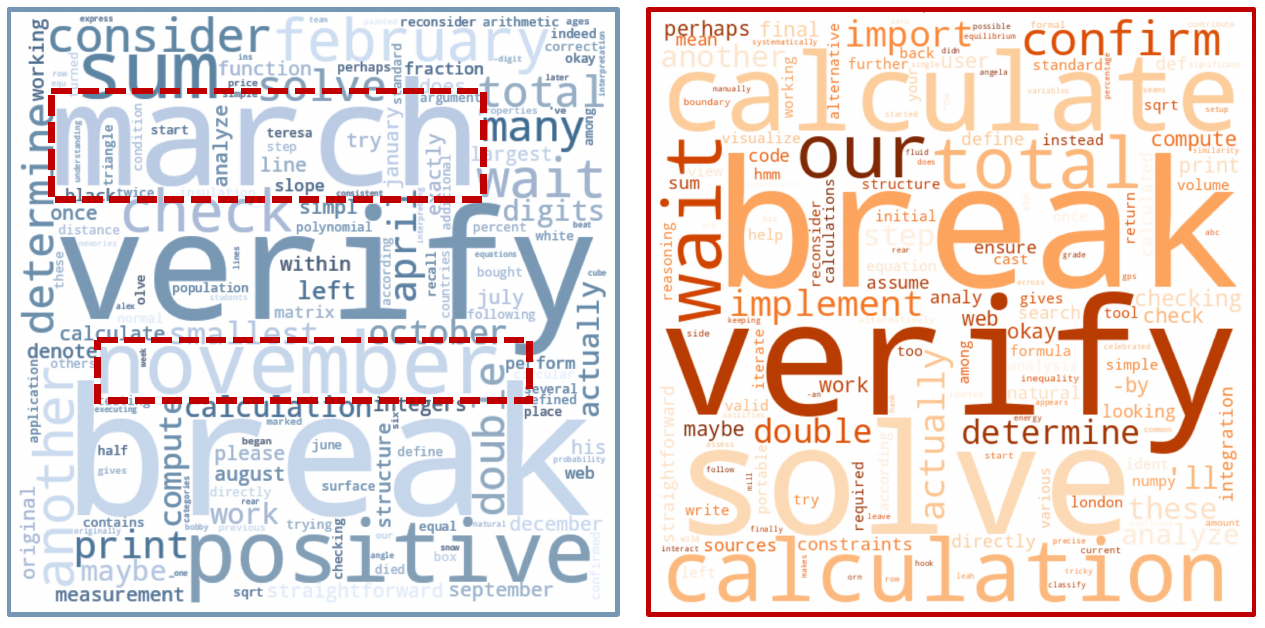}
\captionof{figure}{The word cloud of \textcolor{morandiblue}{\textbf{high-entropy}} tokens and those selected by our \textcolor{morandired}{\textbf{BS}} metric.} 
\label{fig:ciyun} 
\setlength\tabcolsep{5pt}  
\end{minipage}
\vspace{-0.6cm}
\end{table}

\vspace{-0.3cm}
\subsection{Scaling Analysis of APPO}

\noindent\textbf{Pass@K Analysis.} 
Figure~\ref{fig:passk} shows that the benefit of APPO extends beyond improving the single best trajectory and instead improves the overall distribution of candidate trajectories, as reflected in the Pass@K results. Across all four datasets, the advantage of APPO is consistently preserved and further enlarged as $k$ increases, indicating that the method improves not only top-1 correctness but also the diversity of valid reasoning paths in the sampling space. On GAIA, Qwen3-14B improves from 43.7 to 46.1 at Pass@1, with the gap further expanding at Pass@5 from 61.2 to 64.0. A similar trend appears on WebWalkerQA , where the gain increases from 40.5 to 42.7 at Pass@1 and from 62.0 to 66.8 at Pass@5. These results suggest that APPO promotes exploration over structurally distinct reasoning trajectories rather than local token-level variations, leading to a broader set of high-quality candidate solutions and significant improvement of pass@K metrics.

\noindent\textbf{Branching Configuration Analysis.} 
Table~\ref{tab:budget} shows that the effectiveness of APPO depends not only on whether branching is introduced, but also on how the rollout budget is allocated between the number of initial trees $N$ and the number of selected tokens $B$. 
Under the same total budget, balanced configurations consistently outperform both extremes. 
For example, when $M=16$, APPO with $(N=4, B=3)$ achieves the best average score of 58.1, exceeding both $(N=8, B=1)$ with 57.9 and $(N=2, B=7)$ with 56.1. 
Increasing $N$ improves the diversity of initial trajectories but leaves less budget for expanding high-impact decision points. 
In contrast, increasing $B$ enables deeper exploration around these decision points but concentrates the budget on fewer initial paths, reducing global coverage. 
APPO performs best in the middle regime because decision-point-guided branching is most effective when the model first explores diverse root trajectories and then expands informative internal decisions.

\noindent\textbf{Ablation on Components.} 
Table~\ref{tab:comp} verifies that all components of APPO contribute to the final gains, with complementary rather than redundant effects. Specifically, replacing the Branching Score (BS) with entropy leads to consistent performance drops of 1.7  and 0.9 points on the two backbones. This suggests that pure entropy fails to prioritize high-impact decision points that are likely to alter downstream reasoning paths. Disabling dual-group advantage estimation also yields a clear degradation, since initial rollouts and branches are generated from different policy distributions and should be compared within their respective groups. Finally, removing $\hat{A}^{\rm fut}$ causes an even larger drop, especially on Qwen2.5-7B, where the average score decreases from 58.1 to 54.7. Overall, BS improves where the model explores, while dual-group advantage estimation and $\hat{A}^{\rm fut}$ ensure that branches are compared under appropriate distributions and assigned more fine-grained credit.
\vspace{-0.3cm}
\subsection{Qualitative Analysis of APPO}

\noindent\textbf{Training Dynamics.} 
Figure~\ref{fig:dyna} compares the training dynamics of APPO and ARPO. 
APPO reaches a higher final reward and follows a more stable improvement trajectory throughout training. 
This advantage appears early and becomes more pronounced as optimization proceeds, especially on Qwen2.5-7B. 
These results suggest that APPO does not merely improve local exploration, but allocates exploration more effectively during training. 
Specifically, APPO branches around high-impact decision points that reflect meaningful differences in reasoning strategy, leading to larger and smoother reward improvements. 
APPO maintains this advantage on both backbones, indicating that the gains mainly come from the algorithmic design rather than model-specific behavior. 

\noindent\textbf{Diversity Analysis.} 
Figure~\ref{fig:julei} presents the DBSCAN~\cite{ester1996density} clustering results of rollouts sampled by ARPO and APPO, leveraging UMAP~\cite{mcinnes2018umap}. 
APPO produces more compact and better-separated clusters than ARPO, whose branches are more diffuse and less structured. 
This suggests that APPO improves diversity at the level of reasoning strategy. 
Branches around similar high-impact decision points remain semantically coherent, while branches from different decision points show larger semantic gaps.  These results indicate that APPO gains from producing branches that are more structurally distinct and therefore more informative for credit assignment.

\noindent\textbf{Interpretation of the BS metric.} We visualize the tokens selected by the pure high-entropy and by our BS metric (dynamically tracked in the experiment and counted finally) in Figure~\ref{fig:ciyun}. While high-entropy tokens do include salient reasoning keywords such as ``verify'', ``sum'', and ``break'', they also contain many rare nouns like ``march'', ``november''. We attribute these cases to \textit{long-tailed effects} of the vocabulary, where uncertainty reflects absolute token rarity rather than reasoning difficulty, and optimizing such tokens is unlikely to produce transferable reasoning gains. In contrast, our BS metric filters out these cases by emphasizing tokens with larger downstream distributional shifts between the current and old policies. As a result, it is more likely to select tokens that actually redirect the reasoning trajectory and steer  the success and failure of the continuations.
 
\vspace{-0.1cm}
\section{Conclusion}

We proposed \textbf{APPO}, an agentic RL algorithm that shifts branching and credit assignment from coarse tool- or workflow-level units to fine-grained decision points in the generated sequence. APPO uses a Branching Score to select high-value branching locations and introduces an extra future-aware advantage scaling for more targeted credit assignment. Experiments on 13 benchmarks show that APPO consistently outperforms strong baselines while keeping efficient tool-calls. Our findings are also generic and enlightening, suggesting that modeling procedural decisions offers a practical direction for improving exploration and credit assignment in agentic RL.

% \section*{References}
% \medskip
{
    \small
    \bibliographystyle{abbrv}
    \bibliography{main}
}

%%%%%%%%%%%%%%%%%%%%%%%%%%%%%%%%%%%%%%%%%%%%%%%%%%%%%%%%%%%%
\clearpage
\newpage 
\appendix

 \begin{center}\Large{\textbf{\appendixname{\\APPO: Agentic Procedural Policy Optimization}}}\end{center}
% \section{Contents}
The Appendix of the paper is organized as:
\begin{itemize}
    \item Appendix \ref{appx:a}: We give the proof of theorems.
    \item Appendix \ref{appx:b}: We introduce datasets and baselines.
    \item Appendix \ref{appx:c}: We report full implementation details.
    \item Appendix \ref{appx:d}: We report alternative designs of the BS metric.
    \item Appendix \ref{appx:e}: We report studies of the branching budget.
    \item Appendix \ref{appx:f}: We report the prompts we used.
    \item Appendix \ref{appx:g}: We report the limitations.
    \item Appendix \ref{appx:h}: We report case studies.
    \item Appendix \ref{appx:i}: We report the algorithm of APPO. 
    \item Appendix \ref{appx:j}: We report the impact statement. 
    \item Appendix \ref{appx:k}: We report the declaration of LLM usage. 
\end{itemize}

\section{Mathematical Proof}
\label{appx:a}

\subsection{Proof of Theorem \ref{theo:1}} 
\label{proof:theorem1}

\textbf{Proof.}
Let $\mathcal{D}=\{D_1,\dots,D_K\}$ denote the candidate branching locations in one rollout, where each $D_i$ is a fine-grained decision point.
Let $n_i$ be the number of branches assigned to $D_i$, with total budget $\sum_{i=1}^K n_i = M$.
The baseline uniform allocation sets $n_i^{\mathrm{uniform}} = M/K$ for all $i$.
By the assumptions, APPO adopts the proportional branch allocation $n_i^{\mathrm{APPO}}$ (assigning more branches to higher-${\rm BS}_i$ locations); top-${\rm BS}$ selection in practice approximates this rule.
For each $D_i$, let $g_i$ denote the gradient estimator from its branch samples.
Assuming i.i.d.\ branch rewards at $D_i$, the total estimator variance decomposes as
\begin{equation}
\mathrm{Var}(g) = \sum_{i=1}^K \mathrm{Var}(g_i \mid D_i)
= \sum_{i=1}^K \frac{\sigma_i^2}{n_i},
\end{equation}
where $\sigma_i^2 := \mathrm{Var}(R \mid D_i)$ as in the theorem.  For i.i.d.\ branch rewards at $D_i$,
$\mathrm{Var}(g_i \mid D_i) = \sigma_i^2/n_i$.
By Lagrange multipliers (or Cauchy--Schwarz), the allocation minimizing
$\sum_i \sigma_i^2/n_i$ subject to $\sum_i n_i = M$ is
$n_i^* = M \sigma_i / \sum_j \sigma_j$.
Therefore, for any uniform allocation $n_i^{\mathrm{uniform}} = M/K$,
\[
\sum_{i=1}^K \frac{\sigma_i^2}{n_i^{\mathrm{APPO}}}
\le
\sum_{i=1}^K \frac{\sigma_i^2}{n_i^{\mathrm{uniform}}}
= \frac{K}{M}\sum_{i=1}^K \sigma_i^2,
\]
with strict inequality unless all $\sigma_i$ are equal.
Setting $\Delta_\Omega({\rm BS})
:= \mathrm{Var}(g_{\mathrm{base}}) - \mathrm{Var}(g_{\mathrm{APPO}}) \ge 0$
completes the proof.

\subsection{Proof of Theorem \ref{theo:2}}
\label{proof:theorem2}
\textbf{Proof.}
Write $\mathcal{J}(\pi)\equiv\eta(\pi)$. By the Performance Difference Lemma,
\begin{equation}\label{eq:pdl2}
\mathcal{J}(\pi_{\rm new})-\mathcal{J}(\pi_{\rm old})
=\frac{1}{1-\gamma}\,
\mathbb{E}_{s\sim\rho^{\pi_{\rm new}},\,a\sim\pi_{\rm new}}
\big[A^{\pi_{\rm old}}(s,a)\big].
\end{equation}
Define the APPO surrogate actually optimized (weighted advantage on the
branching mixture distribution $\rho^q$):
\begin{equation}
L_{\rm APPO}(\pi_{\rm new})
=\frac{1}{1-\gamma}\,
\mathbb{E}_{s\sim\rho^{q},\,a\sim\pi_{\rm new}}
\big[\omega(s)\,A^{\pi_{\rm old}}(s,a)\big],
\quad \omega(s)=1+b\,\hat{A}^{\rm fut}(s).
\end{equation}
Then:
\begin{equation}\label{eq:x}
\begin{split}
   & \mathcal{J}(\pi_{\rm new})-\mathcal{J}(\pi_{\rm old})-L_{\rm APPO}(\pi_{\rm new})
= \\
&\underbrace{\frac{1}{1-\gamma}\Big(
\mathbb{E}_{\rho^{\pi_{\rm new}},\pi_{\rm new}}-\mathbb{E}_{\rho^{q},\pi_{\rm new}}
\big)[A^{\pi_{\rm old}}]\Big]}_{\text{occupancy mismatch}}
+\underbrace{\frac{1}{1-\gamma}\,
\mathbb{E}_{\rho^{q},\pi_{\rm new}}
\big[(1-\omega(s))A^{\pi_{\rm old}}(s,a)\big]}_{\text{weighting mismatch}}.
\end{split} 
\end{equation}
Since $|A^{\pi_{\rm old}}(s,a)|\le r_{\max}/(1-\gamma)$ and
$\|\omega\|_\infty\le 1+b(1+\epsilon')$, the simulation lemma
(e.g., TRPO analysis~\cite{schulman2017proximal,agarwal2019reinforcement}) yields
\begin{equation}
\Big|\mathbb{E}_{\rho^{\pi_{\rm new}},\pi_{\rm new}}[A^{\pi_{\rm old}}]
-\mathbb{E}_{\rho^{q},\pi_{\rm new}}[A^{\pi_{\rm old}}]\Big|
\le \frac{2\gamma r_{\max}}{(1-\gamma)^2}\,
\mathbb{E}_{s\sim\rho^{\pi_{\rm new}}}
\mathrm{TV}\big(\pi_{\rm new}(\cdot|s),\pi_{\rm old}(\cdot|s)\big).
\end{equation}
Under $\mathrm{KL}(\pi_{\rm new}\|\pi_{\rm old})\le\epsilon$,
Pinsker's inequality bounds the TV term by $O(\sqrt{\epsilon})$;
absorbing policy and occupancy errors into a single constant $C$ gives the following equation:
\begin{equation}
\mathcal{J}(\pi_{\rm new})-\mathcal{J}(\pi_{\rm old})
\ge L_{\rm APPO}(\pi_{\rm new})
-\frac{C\epsilon}{(1-\gamma)^2}.
\end{equation} 
 
\hfill$\square$
\section{Datasets and Baselines.}\label{appx:b}
The datasets employed in our experiments are introduced as the following:

\noindent\textbf{AIME24.} AIME24 comprises the 30 problems from the 2024 American Invitational Mathematics Examination (AIME), split across two competition papers. Each problem demands an integer answer between 0 and 999, yet the reasoning paths required are far from trivial—contestants must navigate number theory, combinatorics, geometry, and algebra. Its small size makes it sensitive to variance, so performance is typically reported as average accuracy over multiple runs.

\noindent\textbf{AIME25.} AIME25 follows the same format as AIME24, providing another 30 competition-level problems released in February 2025. Because the problems post-date most training cutoffs, AIME25 serves as a relatively contamination-resistant probe of genuine mathematical reasoning. It has quickly become a standard checkpoint for evaluating frontier models on olympiad-style problem solving.

\noindent\textbf{MATH500.} MATH500~\cite{lightman2023let} is a 500-problem subset drawn from the broader MATH to span all seven subject categories and five difficulty levels of the original collection. Problems range from introductory algebra to competition-level precalculus, making the subset a compact yet representative testbed for mathematical reasoning. It is widely used as an in-distribution evaluation complement to harder olympiad benchmarks.

\noindent\textbf{MATH.} MATH~\cite{hendrycks2021measuring} contains 12,500 problems sourced from high school mathematics competitions, covering algebra, counting and probability, geometry, number theory, and precalculus, among others. Each problem is accompanied by a step-by-step solution, enabling process-level as well as answer-level evaluation. The dataset is notable for its difficulty gradient: even the easiest problems require multi-step symbolic reasoning, while the hardest rival olympiad-level challenges.

\noindent\textbf{GSM8K.} GSM8K is released by~\cite{cobbe2021training} at OpenAI and consists of 8,500 linguistically diverse word problems pitched at the elementary school level. Despite their apparent simplicity, the problems require chains of two to eight arithmetic steps, making them a reliable probe of basic multi-step reasoning. The dataset is split into 7,500 training and 1,000 test examples and remains one of the most widely cited benchmarks for evaluating arithmetic reasoning in language models.

\noindent\textbf{HotpotQA.} HotpotQA~\cite{yang2018hotpotqa} contains roughly 113,000 Wikipedia-based question–answer pairs that are explicitly designed to require reasoning over two supporting documents. A key design feature is the inclusion of sentence-level supporting-fact annotations, which allow evaluation of both answer correctness and the quality of the reasoning chains.

\noindent\textbf{2WikiMultihopQA.} 2WikiMultihopQA~\cite{ho2020constructing} constructs multi-hop questions by combining information from two Wikipedia articles, using structured Wikidata triples to guarantee that each question genuinely requires cross-document reasoning rather than being solvable from a single passage. The dataset provides explicit evidence chains alongside each question, enabling fine-grained evaluation of intermediate reasoning steps. 
 
\noindent\textbf{MuSiQue.} MuSiQue~\cite{trivedi2022musique} takes a bottom-up approach to constructing multi-hop questions: single-hop questions from existing datasets are composed into 2–4 hop chains via directed acyclic graphs, ensuring that each hop is genuinely necessary and that shortcuts are systematically eliminated. The dataset contains approximately 25,000 questions and is considered harder to game than earlier multi-hop benchmarks, where models could often answer correctly by attending to a single passage.

\noindent\textbf{Bamboogle.} Bamboogle~\cite{press2023measuring} is a small, hand-crafted collection of 125 two-hop factual questions, assembled specifically by filtering out any question that Google's search engine answers correctly. This adversarial construction criterion means the questions tend to require non-obvious compositional reasoning that surface-level retrieval struggles with. Despite its modest size, Bamboogle is frequently used as a stress test for retrieval-augmented generation systems.

\noindent\textbf{GAIA.} GAIA~\cite{mialon2023gaia} is a benchmark for General AI Assistants comprising 466 real-world questions across three difficulty levels. What distinguishes GAIA is its emphasis on practical, tool-assisted problem solving: questions may require web browsing, file parsing, multi-modal understanding, and multi-step planning in combination. Humans solve the benchmark with high accuracy, yet agents still struggle significantly, positioning GAIA as a meaningful frontier for agentic evaluation.

\noindent\textbf{HLE}. Humanity's Last Exam (HLE)~\cite{phan2025humanity} consists of 2,500 questions assembled by over 1,000 subject-matter experts to represent the frontier of human knowledge. Questions span mathematics, the natural sciences, humanities, and professional domains, with roughly 10\% requiring image comprehension. The benchmark was explicitly designed to resist saturation: at release, top models scored below 10\%, making it one of the most demanding closed-ended evaluations currently available.

\noindent\textbf{WebWalker.} WebWalker~\cite{wu2025webwalker} benchmarks LLMs on web traversal. The accompanying dataset, WebWalkerQA, contains 680 questions derived from real websites, where answering correctly requires an agent to plan a sequence of click actions across multiple pages. The benchmark highlights a gap between static retrieval and dynamic, navigation-based information seeking.

\noindent\textbf{XBench.} XBench~\cite{chen2025xbench} is a profession-aligned, dynamic evaluation suite introduced by Sequoia Capital to measure AI agent productivity in real-world occupational contexts. Rather than testing isolated capabilities, it presents tasks drawn from domains such as marketing, software engineering, and legal work, scored against expert-produced reference outputs. Its evergreen design is intended to resist benchmark saturation and track whether agent performance translates into genuine workplace utility.

The baselines compared in our experiments are shown in the following:

\noindent\textbf{GRPO.} GRPO~\cite{guo2025deepseek} is critic-free, which groups multiple responses per prompt and utilizing relative reward comparisons within these groups, it optimizes policies based on intra-group performance rather than absolute scalar values, thereby improving sample efficiency.

\noindent\textbf{Reinforce++.}  REINFORCE++~\cite{hu2025reinforce++} is an enhanced variant of the classical REINFORCE algorithm that incorporates key optimization techniques from PPO—including token-level KL penalties and reward normalization—while eliminating the need for a critic model. It achieves three primary objectives: simplicity, enhanced training stability, and reduced computational overhead, making it a practical and efficient baseline for RLHF-based LLM alignment.

\noindent\textbf{DAPO.} DAPO~\cite{yu2025dapo}  is an open-source, industrial-scale RL system for LLM post-training proposed by ByteDance. It introduces four key techniques—Clip-Higher, dynamic sampling, token-level policy gradient loss, and overlong reward shaping—to address entropy collapse and training instability commonly observed in GRPO-based long chain-of-thought training.

\noindent\textbf{GPPO.} GPPO~\cite{su2025klear}  addresses the gradient vanishing problem caused by hard token-level clipping in standard PPO-style training. By gently reintroducing bounded gradients from clipped tokens, GPPO enables finer-grained exploration control and more stable policy updates, serving as the foundation for the entropy-regularized extension CE-GPPO.

\noindent\textbf{CISPO.} CISPO~\cite{chen2025minimax} is a reinforcement learning algorithm proposed in the MiniMax-M1 project that clips importance-sampling weights at the sequence level rather than applying per-token update clipping. This design allows all tokens—including low-probability ones—to contribute to gradient updates, reducing variance and improving training stability in off-policy LLM fine-tuning scenarios.

\noindent\textbf{GIGPO.} GiGPO~\cite{feng2025group} is a critic-free RL algorithm designed for training long-horizon LLM agents, proposed by researchers at Nanyang Technological University. It introduces a two-level grouping structure that combines episode-level advantage estimation with step-level credit assignment, enabling fine-grained attribution of rewards to individual actions without requiring an additional value network.

\noindent\textbf{ARPO.} ARPO~\cite{dong2025agentic} is a reinforcement learning framework specifically designed for multi-turn, tool-augmented LLM agents. It addresses the entropy increase observed after tool interactions by introducing entropy-based adaptive sampling and fine-grained advantage attribution at the action level, enabling more effective exploration and stable policy optimization in agentic settings.

\noindent\textbf{Qwen3.} Qwen3~\cite{yang2025qwen3} is the latest generation of the Qwen large language model family, released in April 2025. It introduces a hybrid thinking mode that allows flexible switching between deliberate chain-of-thought reasoning and fast non-thinking inference, supports 119 languages, and was trained on a corpus exceeding 36 trillion tokens, achieving competitive performance across a broad range of benchmarks.

\noindent\textbf{DeepSeek-R1.} DeepSeek-R1~\cite{guo2025deepseek} is a large reasoning model developed by DeepSeek that incentivizes complex reasoning capabilities through large-scale reinforcement learning with verifiable rewards, without relying on supervised chain-of-thought data in its initial training stage. The resulting model exhibits emergent reasoning behaviors such as self-reflection and dynamic strategy adjustment, achieving performance competitive with OpenAI's o1 series on mathematical and scientific benchmarks.

\noindent\textbf{QwQ.} QwQ~\cite{team2024qwq} is an open-source reasoning model series developed by the Qwen team, first released in November 2024, designed to tackle problems requiring deep analytical thinking. The flagship QwQ-32B model demonstrates that a mid-sized model can achieve competitive performance against state-of-the-art reasoning systems such as DeepSeek-R1 and OpenAI o1-mini through extended chain-of-thought reasoning.

\noindent\textbf{GPT-4o.} GPT-4o~\cite{hurst2024gpt} is OpenAI's flagship multimodal generative model released in May 2024. It achieves significantly lower latency than its predecessors by natively handling cross-modal inputs without modality-specific preprocessing pipelines, establishing a new standard for real-time multimodal AI interaction.

\noindent\textbf{RAG.} RAG~\cite{lewis2020retrieval} enhances large language model outputs by dynamically retrieving relevant documents from an external knowledge base at inference time, thereby grounding generation in up-to-date and domain-specific information. By decoupling parametric knowledge stored in model weights from non-parametric knowledge retrieved on demand, RAG substantially reduces hallucination and improves factual accuracy without requiring full model retraining.

\noindent\textbf{Search-o1.} Search-o1~\cite{li2025search} augments large reasoning models with an agentic search workflow, enabling them to dynamically retrieve external knowledge when they encounter uncertain or knowledge-intensive steps during long chain-of-thought reasoning. By seamlessly interleaving retrieval with reasoning rather than treating search as a preprocessing step, Search-o1 improves both the reliability and factual grounding of complex multi-step inference.

\noindent\textbf{WebThinker.} WebThinker~\cite{li2025webthinker} is a deep research agent that empowers large reasoning models to autonomously search the web, navigate across multiple web pages, and synthesize retrieved information into structured research reports. It introduces a Deep Web Explorer module that tightly integrates web interaction with chain-of-thought reasoning, enabling end-to-end autonomous research on complex, knowledge-intensive tasks.

\noindent\textbf{ReAct.} ReAct~\cite{yao2022react} is a prompting paradigm that interleaves verbal reasoning traces with executable actions in large language models, enabling them to dynamically plan, retrieve information, and adjust their behavior based on environmental feedback. By synergizing chain-of-thought reasoning with tool use, ReAct substantially improves interpretability and task performance on knowledge-intensive decision-making benchmarks compared to reasoning-only or acting-only baselines.

\section{Implementation Details}\label{appx:c}
Our APPO follows a training pipeline largely consistent with ARPO. Specifically, we first train the backbone model to obtain fundamental tool-use capabilities using ToolStar's 54K SFT dataset together with an additional 0.8K samples drawn from the STILL dataset. For SFT training, we employ LLAMA-Factory~\cite{zheng2024llamafactory} with a learning rate of 7e-6, DeepSpeed ZeRO-3~\cite{rasley2020deepspeed}, and Flash-Attention2~\cite{dao2023flashattention}. The batch size is set to 128, the weight decay is 0.1, and the model is trained for 3 epochs. We adopt BF16 mixed precision and set the maximum input length to 4096 tokens.

For the RL stage, we implement APPO based on the VERL framework~\cite{sheng2025hybridflow}. In particular, all tool execution results are masked out from the loss computation to avoid introducing bias toward tool outputs. Accordingly, the loss is computed only over tokens corresponding to textual reasoning and tool calls. We use different configurations for DeepReasoning Tasks and Deep Search Tasks:

(i) DeepReasoning Tasks: For 7B models,  our default configuration adopts a total training batch size of 128, a PPO mini-batch size of 16, a global rollout size of 16, and an initial sampling size of 4. The maximum response length for each interaction is limited to 4096 tokens. For APPO rollouts, we set the entropy coefficient to 0.2, parameter $b$ to 0.5, and the threshold to 0.5. The decay factor gamma is set to $2^{-\frac{1}{32}}$. For clipping with coefficient $\epsilon$, in implementation we adopt a non-symmetric setting  [1,1.2], to encourage the positive exploration of the future value term.  To improve training stability, the KL divergence coefficient in GRPO is fixed at 0. The RL stage is carried out for 2 epochs on 8 NVIDIA H100 GPUs.

(ii) DeepSearch Tasks: For 8B models, we use the same configuration as the above, except that the maximum response length for each interaction is increased to 8192 tokens. For 14B models, the same hyperparameter settings are retained, while the experiments are also conducted on 16 NVIDIA H100 GPUs. Since the dataset contains only 1K samples, the RL stage is performed for 5 epochs.

Regarding the significance analysis, our approach adheres to established conventions by reporting dataset-level averages as well as pass@3 and pass@5 results. These experimental settings avoid concerns regarding statistical significance.

\begin{table}[t]
\caption{Analysis on branching configuration with Qwen2.5-7B-Instruct. The best results are in \textbf{bold}.}
\centering
\label{tab:appx} 
\renewcommand{\arraystretch}{1.1}   
\setlength\tabcolsep{6pt} 
\resizebox{1\columnwidth}{!}{
\begin{tabular}{lcccccc}
\hline
\textbf{Method} & \textbf{WebWalker} & \textbf{HotpotQA} & \textbf{2Wiki} & \textbf{Musique} & \textbf{Bamboogle} & \textbf{Average} \\ \hline
APPO $(N=1,   B=2,L=3,M=16)$ & 28.5 & {63.0} & {78.0} &  {29.8} &  {71.0} & 54.1 \\
APPO $(N=1,   B=4,L=2, M=16)$ & 28.5 & 63.8 & 78.8 & 30.2 & 71.3 & 54.5 \\
APPO $(N=3,   B=1,L=2,M=16)$ & \textbf{33.2} & \textbf{64.9} & \textbf{81.0} & \textbf{31.5} & \textbf{77.8} & \textbf{57.7} \\ \hline
\end{tabular}}\vspace{-0.4cm}
\end{table}

\begin{figure*}[t] 
    \centering
    \includegraphics[width=1\linewidth]{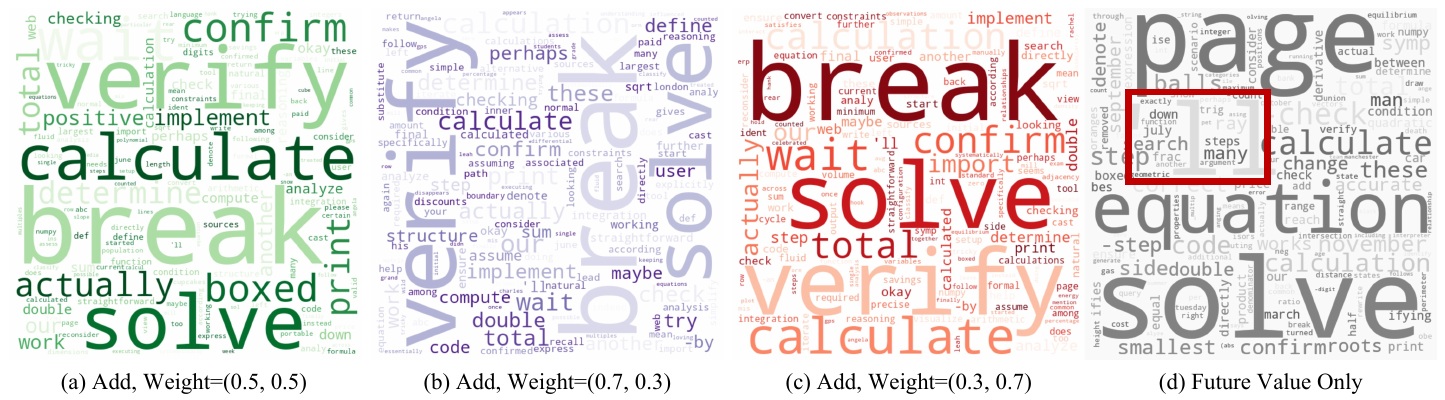} 
    \vspace{-0.4cm}
    \caption{WordCloud of tokens selected by alternative designs of the BS metric.} 
    \vspace{-0.5cm}
    \label{fig:altering}
\end{figure*} 

\section{Alternative Designs of the BS metric}\label{appx:d} 

In this section, we investigate different formulations of the BS metric, including additive combinations of normalized entropy and future value with varying weights, as well as the case utilizing future value alone. Figure \ref{fig:altering} presents the word clouds obtained under these four settings. We observe that the additive design successfully captures tokens significant to reasoning, such as ``calculate'', ``verify'', ``break'', and ``solve''. Interestingly, if the overall metric is disproportionately biased towards future value, the model captures special tokens like \texttt{``'ll''} (not very clear in the Figure \ref{fig:altering}.d. We enclose it in a box, just below the word ``page''). We attribute these tokens to positions where the model reaffirms existing conclusions; while they may have high influence on subsequent rollouts, they fail to reflect actual value for model training and fine-grained supervision.

\section{More Sensitivity Analysis of Key Hyper-parameters}\label{appx:e}
\paragraph{Studies of the number of the branching loop $L$.} In the main paper, our experiments are limited to a single-layer rollout tree, where all branching operations are applied only to the initial rollout. Consider a multi-layer rollout tree, i.e., the case where $L>1$, we have $M=B\cdot(N+1)^L$. The results of our sensitivity analysis are presented in Table \ref{tab:appx}.

We observed that: (i) When $N=1$, the overall performance of the model is relatively poor, showing even a slight decline compared to the case in the main paper where $L=1$ and $M$ is even smaller. For example, with $(N=1, B=2, L=3, M=16)$, the overall performance is approximately 54.1\%, which is about 1.1 points lower than the case $(N=2, B=3, L=1, M=8)$ reported in Table \ref{tab:budget}. We attribute this to the insufficient diversity of the initial rollout, which causes the model to be significantly affected by randomness, leading to performance fluctuations. (ii) When the initial rollout is kept the same, varying $B$ or $L$ does not have a noticeable impact on performance. We believe this is because the two settings largely share the prefix information of the rollouts, resulting in substantial overlap in the sampling distribution. These results once again confirm that our choice in the main text is an optimal trade-off.

% \paragraph{Studies of the Key Hyper-parameter $b$.}

\section{Prompts}\label{appx:f} 
Following ~\cite{dong2025agentic}, The prompt used in APPO is listed in the following:

\begin{tcolorbox}[
  colback=gray!10,     % 背景灰（想要透明就改成 colback=white 或去掉）
  colframe=black,      % 外框颜色
  boxrule=1pt,         % 线宽
  arc=2mm,             % 圆角半径
  left=4pt,right=4pt,top=4pt,bottom=4pt, % 内边距
] 

 \textit{You are a helpful assistant that can solve the given question step by step with the help of the wikipedia search tool and python interpreter tool.  
Given a question, you need to first think about the reasoning process in the mind and then provide the answer. 
During thinking, you can invoke the wikipedia search tool to search and python interpreter tool to calculate the math problem for fact information about specific topics if needed. 
The reasoning process and answer are enclosed within <think> </think> and <answer> </answer> tags respectively,  
and the search query and result are enclosed within <search> </search> and <result> </result> tags respectively.  
For example, <think> This is the reasoning process. </think> <search> search query here </search> <result> search result here </result>  
<think> This is the reasoning process. </think> <python> python code here </python> <result> python interpreter result here </result> 
<think> This is the reasoning process. </think> <answer> The final answer is  \boxed{answer here}  </answer>.  
In the last part of the answer, the final exact answer is enclosed within \boxed{} with latex format. }
 
\end{tcolorbox}

\section{Limitations}\label{appx:g} 
Although the proposed APPO introduces innovations in both how to branch and where to branch, our work still has the following limitations: (i) The splitting point selection method is validated solely through experiments, lacking theoretical guarantees that BS is an optimal branching criterion. However, fully quantifying the actual branching value of a specific point requires a more systematic framework design and an exploration of the intrinsic properties of LLMs. (ii) Following prior work, tools employed by our APPO are currently restricted to Search and Python, and have not been extended to other application-oriented tools.
Despite these factors, APPO demonstrates sufficient performance advantages across the broadest possible range of experimental settings.

\section{Case Study}\label{appx:h}
We provide two kinds of cases in this section: (i) the branching stage cases, where the initial rollout is wrong but turns correct by our branching selections; (ii) the inference cases of ARPO and APPO, where ARPO fails but solvable to our APPO.   

\section{Algorithm}\label{appx:i}
The algorithm of APPO is shown in Algorithm \ref{alg:appo}.
\DontPrintSemicolon
\SetKwInOut{Input}{Input}
\SetKwProg{Fn}{Function}{}{end}
\SetKwComment{tcp}{// }{}
\SetAlCapSkip{0.5em}
\SetAlgoNlRelativeSize{-1}
\SetAlFnt{\small}
\SetKw{KwTo}{to}
 
\begin{algorithm}[t]
\caption{APPO Training Pipeline}
\label{alg:appo}
\Input{policy $\pi_\theta$, reference policy $\pi_{\rm ref}$, toolset $T$, training set $\mathcal{D}$, rollout budget $M$, number of initial rollouts $N$, number of selected branching points $B$, PPO epochs $K_{ppo}$, clipping thresholds $\epsilon,\epsilon'$, procedural weight $b$, decay factor $\gamma$}

\For{each training step}{
    sample input $x \sim \mathcal{D}$\;
    set behavior policy $\pi_{\rm old} \leftarrow \pi_\theta$\;
    $\mathcal{T}_{init} \leftarrow \emptyset$\;

    \tcp{Initialization: generate initial rollouts}
    \For{$n=1$ \KwTo $N$}{
        generate a full rollout $\mathcal{H}_n \sim \pi_{\rm old}(\cdot \mid x;T)$ through agent-environment interaction\;
        add $\mathcal{H}_n$ to $\mathcal{T}_{init}$\;
    }

    \tcp{Mini-batch procedural branching and optimization}
    \For{$e=1$ \KwTo $K_{ppo}$}{
        $\mathcal{T}_{branch} \leftarrow \emptyset$\;

        \While{$|\mathcal{T}_{branch}| < M-N$}{
            uniformly sample one rollout $\mathcal{H}_n$ from the current rollout trees\;

            \For{each valid token position $i$ in $\mathcal{H}_n$}{
                compute token entropy $H_{n,i}$ by Eq.~(3)\;
                compute future value $\Omega_{n,i}$ by Eq.~(4)\;
                compute Branching Score ${\rm BS}_{n,i}$ by Eq.~(5)\;
            }

            select the top-$B$ tokens according to ${\rm BS}_{n,i}$ and denote them as $\mathcal{B}_n$\;

            \ForEach{$i \in \mathcal{B}_n$}{
                resample a continuation from prefix $\mathcal{H}_{n,<i}$ using the current policy $\pi_\theta$ to obtain a branch $\mathcal{H}_{n}^{new}$\;
                add $\mathcal{H}_{n}^{new}$ to $\mathcal{T}_{branch}$\;
                update the rollout tree with $\mathcal{H}_{n}^{new}$\;

                \If{$|\mathcal{T}_{branch}| = M-N$}{
                    \textbf{break}\;
                }
            }
        }

        \tcp{Dual-group advantage estimation}
        evaluate reward $R(\mathcal{H})$ for each rollout $\mathcal{H} \in \mathcal{T}_{init} \cup \mathcal{T}_{branch}$\;
        compute group-relative advantages separately for $\mathcal{T}_{init}$ and $\mathcal{T}_{branch}$ by Eq.~(6)\;

        \tcp{Future-aware procedural advantage}
        \For{each token position $i$ in each initial rollout $\mathcal{H}_n \in \mathcal{T}_{init}$}{
            compute $\hat{A}^{\rm fut}_{n,i}$ by Eq.~(7)\;
            set $\hat{A}_{n,i} \leftarrow \hat{A}^{\rm base}_{n,i}(1+b\cdot \hat{A}^{\rm fut}_{n,i})$\;
        }

        \tcp{Policy optimization on initial rollouts}
        optimize Eq.~(8) on $\mathcal{T}_{init}$ with KL regularization to $\pi_{\rm ref}$\;
        update $\theta$\;
    }
}
\end{algorithm}

\section{Impact Statement}\label{appx:j}
This paper introduces APPO, a novel Agent RL method designed to push the boundaries of where to branch and how to branch. The proposed APPO holds broad societal value, playing a significant role in domains such as search systems, AI-assisted healthcare and education. It is particularly worth noting that, as a work aimed at advancing agent capabilities, APPO not only delivers broader efficacy improvements but also catalyzes increased automated behaviors. Like any research in this field, we are fully committed to participating in regulatory compliance and safety governance for agents. Overall, our proposed APPO offers extensive value to both the research and application communities, fostering healthy and sustainable development in the field of agents.

\section{Declaration of LLM Usage}\label{appx:k}
LLMs are only for polishing the writing of this paper.

\begin{figure*}[h]
\centering
  
\begin{tcolorbox}[
    colback=gray!5!white, 
    colframe=brown!75!brown, 
    title=Case-1 from MATH,
    boxrule=0.3mm, 
    width=\linewidth,  
    arc=2mm, 
    auto outer arc=true
]
The 7th and 8th grades have enrollments of 520 and 650 respectively. The two grades have a total of 18 representatives on the Student Council. How many representatives should the 8th grade have if there is to be a fair representation of the two grades?
\end{tcolorbox}

\vspace{0.2cm}   
 
\begin{minipage}[t]{0.48\linewidth} 
    \begin{tcolorbox}[
        colback=gray!5!white, 
        colframe=red!75!brown, 
        title=Wrong Rollout, 
        boxrule=0.3mm, 
        width=\linewidth, 
        arc=2mm, 
        height=18cm,
        auto outer arc=true,
        top=2mm, bottom=2mm, left=2mm, right=2mm 
    ]
To determine the number of representatives the 8th grade should have, we need 
to ensure that the number of representatives is proportional to the enrollment 
of each grade. Here are the steps to solve the problem:

1. Calculate the total enrollment:
   Total enrollment = 520 + 650 = 1170

2. Determine the proportion of the 8th grade's enrollment to the total enrollment:
   Proportion of 8th grade = 650 / 1170

3. Simplify the proportion:
   To simplify 650/1170, we find the greatest common divisor (GCD) of 650 and 
   1170. The GCD is 50.
   
   650 ÷ 50 / 1170 ÷ 50 = 13 / 23.4   

4. Express the simplified proportion:
   13 / 23.4

5. Calculate the number of representatives for the 8th grade:
   Number of representatives = 13 / (13 + 23.4) × 18

6. Further calculation:
   13 / 36.4 × 18

7. Perform the multiplication:
   13 / 36.4 $\approx$ 0.35714
   0.35714 × 18 $\approx$ 6.42852

8. Round to the nearest whole number:
   Since the number must be a whole number, we round 6.42852 to 6.

\textcolor{red}{\textbf{Thus}}, the 8th grade should have \boxed{6} representatives.
\end{tcolorbox}
\end{minipage} 
\hspace*{\fill}
\begin{minipage}[t]{0.48\linewidth} 
    \begin{tcolorbox}[
        colback=gray!5!white, 
        colframe=blue!75!brown, 
        title=APPO Branching, 
        boxrule=0.3mm, 
        width=\linewidth,  
        arc=2mm, 
        height=18cm,
        auto outer arc=true,
        top=2mm, bottom=2mm, left=2mm, right=2mm 
    ]
To determine the number of representatives the 8th grade should have, we need 
to ensure that the number of representatives is proportional to the enrollment 
of each grade. Here are the steps to solve the problem:

1. Calculate the total enrollment:
   Total enrollment = 520 + 650 = 1170

2. Determine the proportion of the 8th grade's enrollment to the total enrollment:
   Proportion of 8th grade = 650 / 1170

3. Simplify the proportion:
   To simplify 650/1170, we find the greatest common divisor (GCD) of 650 and 
   1170. The GCD is 50.
   
   650 ÷ 50 / 1170 ÷ 50 = 13 / 23.4   

4. Express the simplified proportion:
   13 / 23.4

5. Calculate the number of representatives for the 8th grade:
   Number of representatives = 13 / (13 + 23.4) × 18

6. Further calculation:
   13 / 36.4 × 18

7. Perform the multiplication:
   13 / 36.4 $\approx$ 0.35714
   0.35714 × 18 $\approx$ 6.42852

8. Round to the nearest whole number:
   Since the number must be a whole number, we round 6.42852 to 6.

\textcolor{red}{\textbf{----------Branching----------}}

Wait, let me reconsider. The GCD of 650 and 1170 should be recalculated.

650 = 2 × 5² × 13
1170 = 2 × 3² × 5 × 13

GCD(650, 1170) = 2 × 5 × 13 = 130

So: 650 / 1170 = 650÷130 / 1170÷130 = 5 / 9

Number of 8th grade representatives = (5/9) × 18 = 10

Thus, the 8th grade should have \boxed{10} representatives.
\end{tcolorbox}
\end{minipage}

\end{figure*}

\begin{figure*}[h]
\centering
 
% --- 1. Query Box ---
\begin{tcolorbox}[
    colback=gray!5!white, 
    colframe=brown!75!brown, 
    title=Case-2 from Musique, 
    boxrule=0.3mm, 
    width=\linewidth,  
    arc=2mm, 
    auto outer arc=true
]
What does the acronym of the organization Danish Football Union is part of stand for?
\end{tcolorbox}

\vspace{0.2cm}   

\begin{minipage}[t]{0.48\linewidth}  
    \begin{tcolorbox}[
        colback=gray!5!white, 
        colframe=red!75!brown, 
        title=Wrong Rollout, 
        boxrule=0.3mm, 
        width=\linewidth, 
        arc=2mm, 
        height=18cm,
        auto outer arc=true,
        top=2mm, bottom=2mm, left=2mm, right=2mm  
    ]
    To provide the accurate answer, I need to first identify the organization  Danish Football Union is a part of and then find out the acronym associated 
    with that organization.

    \textcolor{mydarkred}{\textbf{<search>}}
    
    search query: What is Danish Football Union organization?
    
    \textcolor{mydarkred}{\textbf{</search>}}
    
    \textbf{<result>}
    
    \textbf{Five pages omitted}.

    \textbf{</result>}

    The search result indicates that the Danish Football Union (DBU) is a member 
    of both FIFA and UEFA. Now, I will find the acronym for UEFA to provide the answer.

    \textcolor{mydarkred}{\textbf{<search>}}
    
    search query: What is UEFA acronym?
    
    \textcolor{mydarkred}{\textbf{</search>}}
    
    \textbf{<result>}
    
    \textbf{Eight pages omitted}.

    \textbf{</result>}

    \textcolor{red}{\textbf{Therefore}}, the acronym of the organization that the Danish Football Union is 
    part of stands for Union of European Football Associations.
    \end{tcolorbox}
\end{minipage} 
\hfill 
\begin{minipage}[t]{0.48\linewidth} 
    \begin{tcolorbox}[
        colback=gray!5!white, 
        colframe=blue!75!brown, 
        title=APPO Branching, 
        boxrule=0.3mm, 
        width=\linewidth,  
        arc=2mm, 
        height=18cm,
        auto outer arc=true,
        top=2mm, bottom=2mm, left=2mm, right=2mm 
    ]
    To provide the accurate answer, I need to first identify the organization 
    Danish Football Union is a part of and then find out the acronym associated 
    with that organization.

    \textcolor{mydarkred}{\textbf{<search>}}
    
    search query: What is Danish Football Union organization?
    
    \textcolor{mydarkred}{\textbf{</search>}}
    
    \textbf{<result>}
    
    \textbf{Five pages omitted}.

    \textbf{</result>}

    The search result indicates that the Danish Football Union (DBU) is a member 
    of both FIFA and UEFA. Now, I will find the acronym for UEFA to provide the answer.

    \textcolor{mydarkred}{\textbf{<search>}}
    
    search query: What is UEFA acronym?
    
    \textcolor{mydarkred}{\textbf{</search>}}
    
    \textbf{<result>}
    
    \textbf{Eight pages omitted}.

    \textbf{</result>}

    \textcolor{red}{\textbf{----------Branching----------}}
    
    The search result shows that the Danish Football Union is a member of both 
    FIFA and UEFA. The question asks for "the organization" (singular), and since 
    FIFA is the global governing body of football while UEFA is only the European 
    body, the more fundamental organization is FIFA.

    \textcolor{mydarkred}{\textbf{<search>}}
    
    search query: What is FIFA acronym?
    
    \textcolor{mydarkred}{\textbf{</search>}}
    
    \textbf{<result>}
    
    \textbf{Nine pages omitted}.

    \textbf{</result>}

    Therefore, the acronym FIFA stands for International Federation of 
    Association Football.
    \end{tcolorbox}
\end{minipage}

\end{figure*}

\begin{figure*}[h]
\centering
 
\begin{tcolorbox}[
    colback=gray!5!white, 
    colframe=brown!75!brown, 
    title=Case-3 (Tool-Star RL Dataset): when the composer of the film Waqt (1965) was born?, 
    boxrule=0.3mm, 
    width=\linewidth,  
    arc=2mm, 
    auto outer arc=true
] 

<think> 

To find out when the composer of the film Waqt (1965) was born, we need to gather relevant information through Wikipedia searches and identify the necessary details. </think>

\textcolor{mydarkred}{\textbf{<search>}}search query: Who composed music for Waq 1965 film? \textcolor{mydarkred}{\textbf{</search>}}

<result>

Page 1: Lyrics and video of songs from Movie / Album : Waqt (1965); Music by: Ravi; Singer (s): Asha Bhosle, Mahendra Kapoor, Manna Dey, Mohammed Rafi; having star cast: Sunil Dutt, Sadhana, Raaj Kumar, …
Page 2: Waqt songs are composed by Ravi Shankar Sharma (Ravi), while Sahir Ludhianvi wrote its lyrics. Check out Waqt songs list with lyrics and music videos below.
Page 3: Ravi (music director) — Ravi Shankar Sharma, better known by his first name Ravi, is an Indian music director, who has composed music for several Hindi and Malayalam films.
Page 4: These included story writer Akhtar Mirza, who was credited as F.A. Mirza, dialogue writer Akhtar-Ul-Iman, and editor Pran Mehra. Album details and songs for Waqt.
Page 5: The film contains the Hindi songs "Hum Jab Simat Ke", sung by Mahendra Kapoor and Asha Bhosle, "Waqt se din aur raat", sung by Mohammed Rafi, "Aage Bhi Jaane Na Tu", sung by Asha Bhosle and …
Page 6: Listen to Waqt (Original Motion Picture Soundtrack) by Ravi, Sahir Ludhianvi \& Traditional on Apple Music. 1965. 21 Songs. Duration: 1 hour 35 minutes. </result> 

<think> 

Having obtained the relevant information, it is clear that Ravi composed music for this film. Now, to ascertain his date of birth, more focused searches would be necessary to obtain accurate information related to him.</think>

\textcolor{mydarkred}{\textbf{<search>}} 

search query: When was Ravi (music director) born?  \textcolor{mydarkred}{\textbf{</search>}}

<result>

Page 1: Ravi was born on 3 March 1926 in Delhi, British India. He was a music artist and composer, known for Gharana (1961), Dilwale Dulhania Le Jayenge (1995) and Khandan (1965).
Page 2: Ravi Shankar Sharma (3 March 1926 – 7 March 2012), often referred to mononymously as Ravi, was an Indian music director, who had composed music for several Hindi and Malayalam films.</result>  

<think> 

Based on the search results found so far, it looks like Ravi (full name Ravi Shankar Sharma) has a birth date which seems to repeat often, but it appears consistent across multiple relevant sources namely March 3, 1926. </think>

\textcolor{mydarkblue}{\textbf{<python>} } 

print("Ravi's complete name is Ravi Shankar Sharma. His birthday is on 3rd March 1926.") \textcolor{mydarkblue}{\textbf{</python>} } 

<result> 

Ravi's complete name is Ravi Shankar Sharma. His birthday is on 3rd March 1926. </result> 

<think> 

Upon confirming the date of birth via the reliable resource, Ravi was born on March 3, 1926 which will correctly identify the composer’s date of birth without any further steps.
</think>

\# Final confirmation
composer\_birth\_date = "March 3, 1926"

\textcolor{mydarkblue}{\textbf{<python>} }

print(f"The composer of the film 'Waqt (1965)', Ravi was born on {composer\_birth\_date}.")

\textcolor{mydarkblue}{\textbf{</python>} }, 

ground\_truth: 3 March 1926
f1\_score: 1.0, 

answer: March 3, 1926, ground\_truth: 3 March 1926
\end{tcolorbox} 

\end{figure*}

\begin{figure*}[h]
\centering
 
\begin{tcolorbox}[
    colback=gray!5!white, 
    colframe=brown!75!brown, 
    title=Case-4 (Tool-Star RL Dataset): what award the director of "The River's End (Film)" might have earned?, 
    boxrule=0.3mm, 
    width=\linewidth,  
    arc=2mm, 
    auto outer arc=true
] 

<think> To find out what award the director of "The River's End (Film)" might have earned, we will start by identifying the director's name from searching on Wikipedia, then proceed further to see if any awards related to him/her can be identified from provided context or additional searches.</think>

<search> search query: Who directed the film 'The River's End'?</search>

<result>
Page 1: Warner Bros. Pictures produced River's End, based on the same source in 1930; Michael Curtiz directed, and Charles Bickford and Evelyn Knapp starred (see AFI Catalog of Feature Films, 1921 …
Page 2: Where to watch JustWatch The River's End 1920 Directed by Victor Heerman, Marshall Neilan

\textbf{six pages omitted.}</result>  

<think> 

It appears there may have slight confusions between different films titled "The River's End". We'll assume that it refers to 'The River's End' directed by Ray Enright which seems consistent across multiple references and is more modern context than others. Now let me try finding director Ray Enright’s awards.</think>  

<search> 

search query: What award did Ray Enright receive as a director?</search> <result>
Page 1: Ray Enright (25 March 1896 – 3 April 1965) was an American film director. He directed 73 films between 1927 and 1953. He was born in Anderson, Indiana and died in Hollywood, California from a heart …
Page 2: In 1926 he started work for Warner Brothers as a cutter and two years later was made a director. The first film he directed was Tracked by the Police (1927), starring Rin Tin Tin.
Page 3: Ray Enright was a director and writer who was born in 1896 in United States and died in 1965 known for Montana, The Spoilers, Gung Ho!, South of St. Louis, Kansas Raiders, Coroner Creek, Return of the …
Page 4: Enright directed Humphrey Bogart in 1941's "The Wagons Roll at Night"; the same year, his film "Teddy The Rough Rider" won an Academy Award for Best Short Subject.
Page 5: Enright, Ray 1 reference stated in Czech National Authority Database mub20191040662 section Nationale Thesaurus voor Auteursnamen ID 256538646 1 reference stated in Virtual International …
Page 6: Note: Names and film titles in the Academy Awards Database are derived from film credits and other documentation submitted to the Academy at the time of nomination or award. Changes to name …
Page 7: Ray Enright American film director / From Wikipedia, the free encyclopedia Ray Enright (March 25, 1896 – April 3, 1965) was an American film director. He directed 73 films between 1927–53, many of them …
</result>  

<think> 

From my search results, it mentions that Ray Enright's film "Teddy The Rough Rider" (which appears to be a short subject film, as indicated by context of award categories) won an Academy Award for Best Short Subject in 1941. Hence, the award should be identified accordingly.</think>

<search> 

search query: Which award did Ray Enright win for "Teddy the Rough Rider"?</search> 

<result>

Page 1: Teddy, the Rough Rider is a 1940 American short drama film directed by Ray Enright. It won an Oscar at the 13th Academy Awards for Best Short Subject (Two-Reel).
Page 2: Enright directed Humphrey Bogart in 1941's "The Wagons Roll at Night"; the same year, his film "Teddy The Rough Rider" won an Academy Award for Best Short Subject.

\textbf{four pages omitted.}

</result>  

<think> 

The information clearly outlines that "Teddy the Rough Rider", directed by Ray Enright, won the Academy Award for Best Short Subject (Two-Reel) at the 13th Academy Awards.</think>  

<python> 
\# Confirm the exact award name
print("The award for 'Teddy the Rough Rider' is Best Short Subject.")
print("Specifically, the category was Best Short Subject (Two-Reel) at the 13th Academy Awards.")

</python>
 
f1\_score: 1.0,

answer: Best Short Subject, ground\_truth: Best Short Subject
 
\end{tcolorbox} 

\end{figure*}

\end{document}